\documentclass{article}

\usepackage{arxiv}

\usepackage[utf8]{inputenc} 
\usepackage[T1]{fontenc}    
\usepackage{hyperref}       
\usepackage{url}            
\usepackage{booktabs}       
\usepackage{amsfonts}       
\usepackage{nicefrac}       
\usepackage{microtype}      
\usepackage{lipsum}		
\usepackage{graphicx}
\usepackage{natbib}
\usepackage{doi}
\usepackage{hyperref}
\usepackage{xcolor}
\definecolor{darkblue}{rgb}{0, 0, 0.5}
\hypersetup{colorlinks=true,citecolor=darkblue, linkcolor=darkblue, urlcolor=darkblue}

\usepackage{algorithm}
\usepackage{algorithmicx}
\usepackage[noend]{algpseudocode}
\algdef{SE}[DOWHILE]{DoWhile}{EndDoWhile}{\algorithmicdo}[1]{\algorithmicwhile\ #1}%

\usepackage{enumitem} 
\usepackage{amsmath} 
\usepackage{amssymb} 
\usepackage{mathrsfs} 
\usepackage{mathtools} 
\usepackage{graphicx}
\usepackage{latexsym}
\usepackage{cognac}
\usepackage{xcolor}
\usepackage{booktabs}
\usepackage{url}
\usepackage{xcolor}
\usepackage[normalem]{ulem}

\usepackage{amsthm}
\theoremstyle{definition}

\usepackage{linguex} 
\usepackage{xspace} 
\newcommand{\mynum}[1]{\textsubscript{#1}}

\newcommand{\link}[1]{{\sc #1}}
\newcommand{\TLEX}{{\sc tlex}\xspace}
\newcommand{\TLINK}{{\sc tlink}\xspace}
\newcommand{\SLINK}{{\sc slink}\xspace}
\newcommand{\ALINK}{{\sc alink}\xspace}
\newcommand{\TLINKs}{{\sc tlink}s\xspace}
\newcommand{\SLINKs}{{\sc slink}s\xspace}
\newcommand{\ALINKs}{{\sc alink}s\xspace}

\usepackage{enumitem}   
\usepackage[detect-none]{siunitx}
\newcommand{\R}{\mathbb{R}}

\newcommand{\N}{\mathbb{N}}
\newcommand{\ignore}[1]{}

\usepackage{tablefootnote}

\newtheorem{theorem}{Theorem}
\newtheorem{definition}{Definition}
\newtheorem{lemma}{Lemma}

\title{TLEX: An Efficient Method for Extracting Exact Timelines from TimeML Temporal Graphs}

\author{Mustafa Ocal \\
	Knight Foundation School of Computing and Information Sciences\\
	Florida International University\\
	\texttt{mocal@fiu.edu} \\
	\And
		{Ning Xie} \\
	Knight Foundation School of Computing and Information Sciences\\
	Florida International University\\
	\texttt{nxie@cis.fiu.edu} \\
 \And
    {Mark Finlayson} \\
	Knight Foundation School of Computing and Information Sciences\\
	Florida International University\\
	\texttt{markaf@fiu.edu} \\
}

\renewcommand{\headeright}{Ocal, Xie, and Finlayson}
\renewcommand{\undertitle}{}
\renewcommand{\shorttitle}{TLEX: TimeLine EXtraction}

\hypersetup{
pdftitle={TLEX: An Efficient Method for Extracting Exact Timelines from TimeML Temporal Graphs},
pdfsubject={TLEX: TimeLine EXtraction},
pdfauthor={Mustafa Ocal, Ning Xie, Mark Finlayson},
pdfkeywords={Timeline, TimeML, Temporal Reasoning, Temporal Information Extraction},
}

\begin{document}
\maketitle

\begin{abstract}
A timeline provides a total ordering of events and times, and is useful for a number of natural language understanding tasks. However, qualitative temporal graphs that can be derived directly from text---such as TimeML annotations---usually explicitly reveal only {\it partial} orderings of events and times. In this work, we apply prior work on solving point algebra problems to the task of extracting timelines from TimeML annotated texts, and develop an exact, end-to-end solution which we call \TLEX (TimeLine EXtraction). \TLEX transforms TimeML annotations into a collection of timelines arranged in a trunk-and-branch structure. Like what has been done in prior work, \TLEX checks the consistency of the temporal graph and solves it; however, it adds two novel functionalities. First, it identifies specific relations involved in an inconsistency (which could then be manually corrected) and, second, \TLEX performs a novel identification of sections of the timelines that have indeterminate order, information critical for downstream tasks such as aligning events from different timelines. We provide detailed descriptions and analysis of the algorithmic components in \TLEX, and conduct experimental evaluations by applying \TLEX to 385 TimeML annotated texts from four corpora. We show that 123 of the texts are inconsistent, 181 of them have more than one ``real world'' or main timeline, and there are 2,541 indeterminate sections across all four corpora. A sampling evaluation showed that \TLEX is 98--100\% accurate with 95\% confidence along five dimensions: the ordering of time-points, the number of main timelines, the placement of time-points on main versus subordinate timelines, the connecting point of branch timelines, and the location of the indeterminate sections. We provide a reference implementation of \TLEX, the extracted timelines for all texts, and the manual corrections to the temporal graphs of the inconsistent texts.
\end{abstract}


\section{Introduction}
A timeline is a data structure that organizes events and times in a total ordering. Timelines are useful for a number of natural language understanding tasks, such as {\bf question answering}, which may require understanding the overall temporal order of events to produce the correct answer~\citep{Saquete:2004}; {\bf cross-document event co-reference (CDEC)} and {\bf cross-document alignment}, whose accuracy can be improved by access to a total ordering of events~\citep{Navarro:2016}; and {\bf summarization} and {\bf visualization}~\citep{Liu:2012}, where timelines can enhance human understanding of the events as well as the temporal structure of the texts.

Unfortunately, timelines are rarely explicit in text, and usually cannot be extracted directly. Instead, texts usually explicitly reveal only {\it partial} orderings of events and times. Such partial information can be used---either through automatic analyzers~\citep{Verhagen:2005}, or manual annotation~\citep{Pustejovsky2003timebank}, or some combination of both---to construct a temporal graph labeled in some temporal representation language, such as a temporal algebra~\citep{Allen:1983} or TimeML~\citep{Sauri:06}.

For temporal graphs labeled with a temporal algebra~\citep{Bartak:2014}, prior work in the field of {\it Qualitative Spatial and Temporal Reasoning} (QSTR) has shown how to extract an exact timeline~\citep{Kreutzmann:2014,Beek:92,Gerevini:95} that represents one possible solution of total ordering. However, this approach does not work for natural language involving temporal information that cannot be encoded in a strictly temporal algebra (i.e., expressions of possible, counterfactual, or conditional worlds---called {\it subordinated} events here), which is the reason why schemes such as TimeML were developed. Also, natural language is often ambiguous, hence it is useful to know which portions of a timeline are {\it indeterminate}---that is, have multiple possible orderings all consistent with the temporal graph. 

There has been some prior work in the field of Natural Language Processing (NLP) on extracting timelines from TimeML graphs using machine learning, but these solutions fall short in dealing with all possible relations and can result in outputs that contain ordering errors (see \S\ref{sec:priorapproaches}). These errors are a natural consequence of using statistical approaches to solve the problem. 

In this work, we apply the theoretical formalism and techniques from QSTR to TimeML annotated texts, formulate timeline extraction as a {\it topological sort problem}, and show that one can achieve a theoretically {\it exact} solution, unlike prior NLP approaches. Furthermore, our framework handles all possible temporal relations (including aspectual and subordination relations) and detects indeterminacy. Specifically, we present \TLEX (TimeLine EXtraction), a method for extracting a set of timelines from a TimeML annotated text, which handles all TimeML relations and achieves perfect accuracy modulo the correctness of an underlying TimeML graph. \TLEX outputs multiple timelines, one for each temporally connected subgraph, including \textit{subordinated} timelines which represent possible (modal), counterfactual, or conditional situations (subordination is a representational feature of TimeML). These subordinated timelines are connected to the main timeline in a trunk-and-branch timeline structure.\footnote{It is worth mentioning that, although \TLEX is the first system that converts a TimeML graph into trunk-and-branch timeline structure, the idea of separating subordinated events in annotations was previously introduced in the multi-axis annotation scheme model~\citep{ning-etal-2018-multi}, in which the authors defined different axis for different types of events in the annotations such as intention/opinion events (on an orthogonal axis), generic events (on a parallel axis), negations (not on an axis), and static events (other axis).} Like prior work, \TLEX checks the TimeML graph for consistency and performs topological sorting to produce a timeline but goes further by automatically identifying inconsistent subgraphs, which allows texts to be manually corrected. It also identifies indeterminacy in ordering, a common consequence of the ambiguity in natural languages.

We provide a formal proof of \TLEX's correctness, as well as experimental evaluations using 385 manually annotated texts comprising 129,860 words across four corpora. This revealed 123 inconsistent temporal graphs, for which we provide manual corrections. We also checked five different features of the extracted timelines using a sampling evaluation. Our measurements show that \TLEX achieves an accuracy of 98--100\% with 95\% confidence on all these measurements. We release (a) a reference implementation of the \TLEX algorithm in Java, (b) the timelines extracted from the corpora, (c) corrections to the corpora, and (d) code to reproduce our experiments\footnote{\url{https://cognac.cs.fiu.edu/jtlex/}}.

There are four main contributions of this work. {\bf First}, we describe the \TLEX algorithm in detail, providing a clear explanation of how it works, a formal proof of its correctness, and a reference implementation. {\bf Second}, we introduce a new indeterminacy marked trunk-and-branch timeline structure as \TLEX's final output. This structure separates main timelines (the ``real world'' or the trunks) from subordinated timelines (the ``possible worlds'' or the branches). This explicit recognition of the separation between ``real world'' and other types of events is important in the semantics of natural language and has not been dealt with in TimeML-to-timeline conversion work before. The identification of indeterminacy in timelines is also novel. {\bf Third}, we demonstrate a new technique for identifying the specific temporal links that contribute to temporal inconsistency, which allows targeted manual correction of inconsistent TimeML annotations. {\bf Last}, we experimentally evaluate \TLEX on four TimeML corpora. Through a sampling evaluation, \TLEX's performance is found to be accurate within error bounds,  confirming our theoretical results. 

The rest of this paper is organized as follows. First, we discuss related work, which we build upon and extend (\S\ref{sec:relatedwork}). Second, we provide a detailed description of \TLEX, breaking the procedure into five main steps, and present proof of correctness and time complexity analysis for each step (\S\ref{sec:approach}). We then describe our experimental evaluations and results (\S\ref{sec:evaluation}). We conclude with a discussion of possible future work (\S\ref{sec:futurework}) and a list of the contributions (\S\ref{sec:contributions}).

\section{Background and Related Work}
\label{sec:relatedwork}

\subsection{Temporal Algebras \& Reasoning}
\label{sec:theory}

\subsubsection{Interval Algebra}
\label{sec:intervalalgebra}

Allen's interval algebra~\citep{Allen:1983} is a calculus for temporal reasoning that defines possible relations between a pair of time intervals and the inference rules for relation compositions. Allen's algebra was one of the first attempts to computationally model temporal knowledge and temporal reasoning. In this formalism, every event is abstracted as a time interval $I$, which comprises a start point ($I^-$) and end time-point ($I^+$), where the start point comes strictly before the end point ($I^- < I^+$). Two intervals can be related by a set of $13$ mutually exclusive \emph{basic} temporal relations: {\it before}, {\it meets}, {\it overlaps}, {\it starts}, {\it during}, and {\it finishes}, their inverses, and {\it equal}. Allen's interval algebra allows any subset of these $13$ basic relations to form a general relation (hence there are in total $2^{13}$ relations); two intervals $A$ and $B$ can be related by a basic relation as $A$ {\it before} $B$, or by a more complex relation such as $A$ \{{\it before, during, starts}\} $B$ when we are less certain about their temporal relation. Allen presented a composition table for composing pairs of temporal relations, which can be used to infer the temporal relation between intervals $A$ and $C$ given the relation between $A$ and $B$ and the relation between $B$ and $C$. A fundamental question in Allen's algebra is the following \emph{satisfiability} problem: given $n$ interval $I_1, \ldots, I_n$ and $m$ pairwise relations among these intervals, does there exist a \emph{configuration of intervals}\footnote{A \emph{configuration} for a finite set of intervals $\mathcal{I}=\{I_1, \ldots, I_n\}$ is a set of $n$ bounded and closed intervals~\cite[[pp.~9]{ligozat13}. More formally, a configuration is a function $C: \mathcal{I} \to \mathbb{R} \times \mathbb{R}$ satisfying the property that, for every $1\leq i\leq n$, if $C(I_i)=\langle a_i, b_i \rangle$ (i.e., if interval $I_i$ is mapped to the closed interval $[a_i, b_i]$) then $-\infty<a_i \leq b_i <\infty$.} that satisfies all the relations? Unfortunately, it was shown that the satisfiability problem is NP-complete~\citep{Vilain:86,Vilain:90}. 

Using Allen's algebra, we can construct a \emph{temporal graph} from texts by viewing events and times as intervals and representing them as the vertices of the graph, and translating the relations between two events/times expressed in natural language into Allen's temporal relations and representing them as the labels of the edges in the graph. Allen's approach is referred to as a {\it qualitative} temporal algebra because it does not represent exact times or durations. A special type of temporal graph called an \emph{atomic} temporal graph, in which all edge labels are basic relations, is of particular importance for our purposes.

In a great deal of later work, Allen's qualitative framework was generalized and extended to the quantitative case, where time-points and durations are given specific metric values. Researchers have proved quite a number of formal results concerning both qualitative and quantitative formalisms (reviewed in \cite{Bartak:2014}). Of particular importance are those results related to checking the \emph{consistency} of temporal graphs, because a timeline can only be constructed for consistent temporal graphs. To check the consistency of graphs constructed using his algebra, Allen developed a polynomial time, constraint-propagation type algorithm that repeatedly uses the composition table to reconcile the relation between every pair of vertices, to ensure that their relation is consistent with those relations imposed by all length-two paths between them (so-called \emph{path consistency}) \citep{Allen:1983}. However, the algorithm does not detect all inconsistencies, and it was later shown that checking the consistency of general temporal graphs with all Allen relations is NP-complete~\citep{Vilain:86,Vilain:90}, which makes the task intractable under the commonly believed $\mathcal{P} \neq \mathcal{NP}$ assumption.

Because the full power of Allen's interval algebra is likely to be intractable, many researchers introduced tractable subclasses of Allen's algebra~\citep{freksa1992temporal,Nebel:95,Vilain:90}. In a celebrated paper~\citep{Nebel:95}, a subclass called ORD-HORN is shown to be the unique maximal tractable subclass of the full Allen algebra: its consistency can be checked in polynomial time by the path-consistency algorithm while checking any subclass that is greater than the ORD-HORN subclass is NP-complete.

\subsubsection{Point Algebra}
\label{sec:pointalgebra}

Introduced by \cite{Vilain:86}, the point algebra (PA) is a symbolic calculus intended to work with qualitative ordering constraints between pairs of time-points; for systematic treatments, see \citep{Bartak:2014} and references therein. PA defines three basic temporal constraints between two time-points: $<$, $>$, and $=$. If there is no information between two time-points $a$ and $b$, then we say $a$ \{{$<$, $>$, $=$}\} $b$. Temporal reasoning in PA is much more efficient since it has only three basic relations. For instance, the CSPAN algorithm checks the consistency of a PA temporal graph in $O(n^2)$ time, where $n$ is the number of time-points in the graph~\citep{Beek:92}. 

An interval sub-algebra of Allen's algebra is called \emph{pointisable} if its relations can be equivalently defined by a PA on the start and end points of the intervals involved~\citep{leeuwenberg2019survey}. For example, it is easy to check that $A$ {\it before} $B$ is equivalent to the constraints imposed by the PA temporal graph in which $A^-,  A^+ , B^-$ and $B^+$ are the nodes and each of the $12$ edges is labeled with an appropriate basic $<$ or $>$ relation (according to the ordering $A^- <  A^+ < B^- < B^+$). Moreover, it can be easily verified that all basic relations are pointisable.  However, one can also show that $A$ \{{\it before, before\_inverse}\} $B$ cannot be represented as PA constraints on the same set of $4$ time-points. In fact, it is known that there are only $167$ pointisable relations among all $2^{13}-1$ non-null relations in Allen's algebra. The translation of the basic Allen's algebra relations into PA relations is shown in Table~\ref{tab:constraints}.

\begin{table}[ht]
\small
\centering
\begin{tabular}{lll} 
\toprule
{\bf Allen's Algebra} & {\bf TimeML} & {\bf Point Algebra\tablefootnote{Here we omit the trivial relations $I^{-} < I^{+}$ for all intervals $I$ involved to save space.}}\\
\midrule
A \link{Before} B  & A \link{before} B       & $(A^+ < B^-)$ \\ 
A \link{After} B & A \link{after} B        & $(B^+ < A^-)$ \\

A \link{Meets} B & A \link{ibefore} B      & $(A^+ = B^-)$ \\ 
A \link{MetBy} B & A \link{iafter} B       & $(B^+ = A^-)$ \\

A \link{Starts} B & A \link{begins} B       & $(A^- = B^-) \land (A^+ < B^+)$ \\ 
A \link{StartedBy} B & A \link{begun\_by} B    & $(A^- = B^-) \land (B^+ < A^+)$ \\

A \link{Finishes} B & A \link{ends} B         & $(B^- < A^-) \land (A^+ = B^+)$ \\ 
A \link{FinishedBy} B & A \link{ended\_by} B    & $(A^- < B^-) \land (A^+ = B^+)$ \\

A \link{During} B & A \link{includes} B     & $(A^- < B^-) \land (B^+ < A^+)$ \\
A \link{DuringInverse} B & A \link{is\_included} B & $(B^- < A^-) \land (A^+ < B^+)$ \\

A \link{Equals} B & A \link{simultaneous} B & $(A^-$ = $B^-) \land (A^+$ = $B^+)$ \\ 
& A \link{identity} B     & $(A^- = B^-) \land (A^+ = B^+)$ \\ 

 & A \link{during} B       & $(B^- = A^-) \land (A^+ = B^+)$ \\ 
 & A \link{during\_inv} B  & $(A^- = B^-) \land (B^+ = A^+)$ \\ 

\midrule

& A \link{initiates} B    & same as A \link{begins} B  \\
& A \link{culminates} B   & same as A \link{ends} B \\
& A \link{terminates} B   & same as A \link{ends} B \\
& A \link{continues} B    & same as A \link{is\_included} B \\
& A \link{reinitiates} B  & same as A \link{is\_included} B \\ \\

\end{tabular}
\caption{Translation of the Allen's Algebra relation and TimeML temporal \& aspectual relations into basic temporal relations between interval start and end points. For an interval $I$ (an event or a time), the start point of the interval is denoted by $I^-$ and the end point is denoted by $I^+$.}
\label{tab:constraints}
\end{table}

Allen's interval algebra and point algebra (PA) are actually subclasses of a more general problem called 
\emph{Temporal Constraint Satisfaction Problems} (TCSPs)~\citep{DMP91}. In a TCSP, time-points are represented as a set of variables, and temporal information is represented by a set of unary or binary constraints over these variables, each specifying a set of permitted intervals. A special case of TCSP, \emph{Simple Temporal Problem} (STP), was also proposed and studied in~\citep{DMP91}. The input to an STP is a set of time-point variables together with constraints that bound the difference between pairs of these variables, and the output is required to be a succinct representation of the set of all possible assignments to time-point variables that meet those constraints. Because STP is solvable in polynomial time and at the same time captures much of the characteristics of temporal constraint problems in real life, it has been extensively studied and applied in a wide range of fields, such as planning, scheduling, robotics, and control theory. For more background on STP and recent extensions, interested readers are referred to~\citep{ORRZ21}.

\subsection{Temporal Information Extraction and Annotation in Text}
\label{sec:schemes}

\subsubsection{TIDES}
Because of the utility of temporal frameworks for reasoning about time, and also the relevance of time to understanding natural language, researchers in Natural Language Processing (NLP) have sought to apply temporal algebras to text understanding. This requires annotation schemes that would allow a person or a machine to annotate a text for the time-points, events, and temporal relations expressed.

Time expressions (\textbf{TIMEXes}) are sequences of tokens (words, numbers, or characters) in text that denote time, including expressions of when something happens, how often something occurs, or how long something takes. Starting from the early 2000s, researchers developed a sequence of \textbf{TIMEX} annotation schemes~\citep{Setzer:2001,Ferro:2001,Pustejovsky:2003:timeml}. This allows the annotation of expressions such as \textit{at 3 p.m.} (when), or \textit{for 1 hour} (how long). Because events are also involved in temporal relations, these approaches were extended into schemes for capturing both times and events. For example, the Translingual Information Detection, Extraction, and Summarization scheme (TIDES) \citep{Ferro:2001} integrates TIMEX2 expressions as well as a scheme for annotating events. TIDES includes annotations for temporal expressions, events, and temporal relations. TIDES uses only six temporal relation types to represent the relations between events, therefore it provides only limited temporal information from texts.

\begin{figure}[t]
\begin{center}
  \includegraphics[width=12cm]{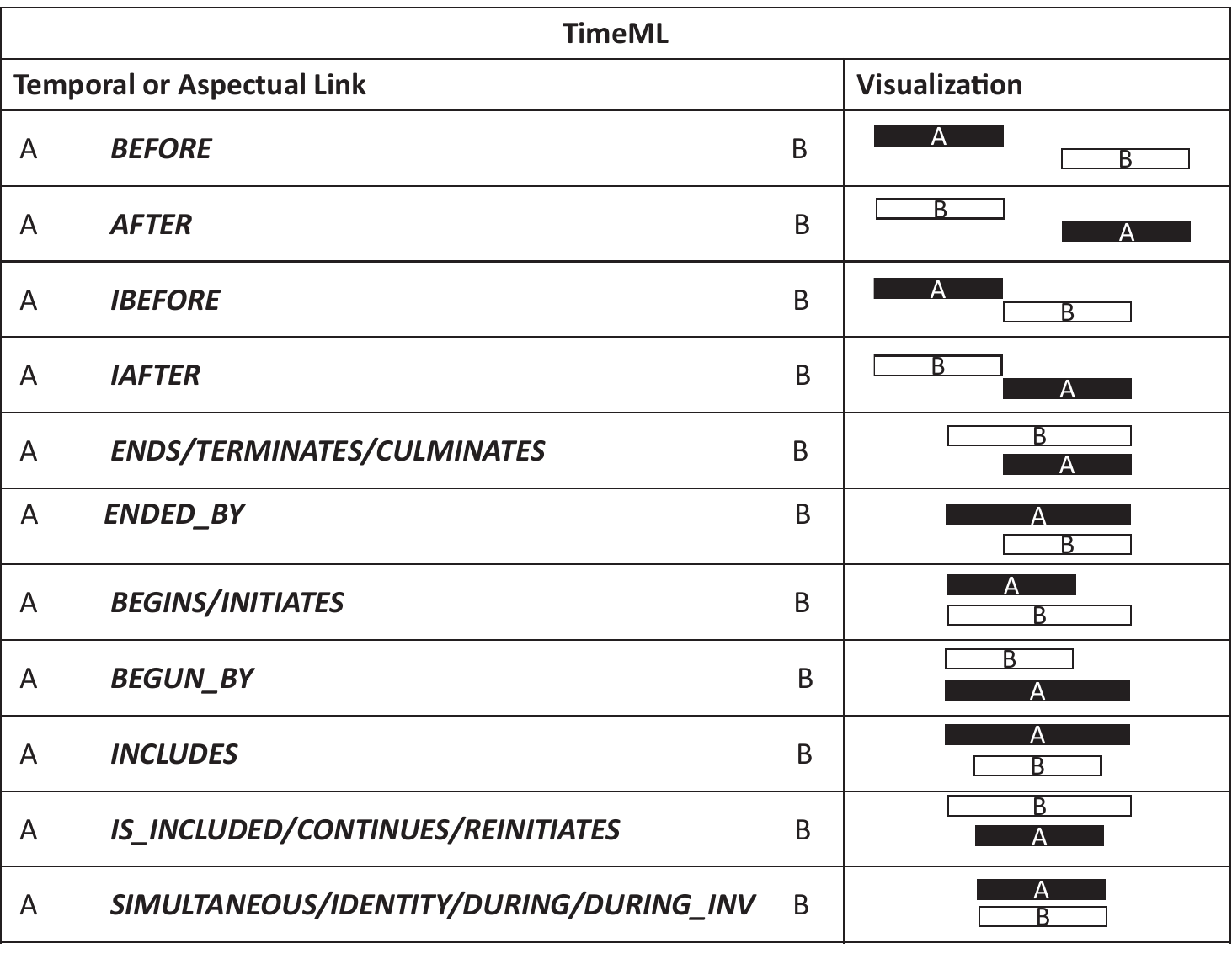}
  \caption{The 14 temporal and 5 aspectual TimeML link types. Temporal links can hold between two events, an event and a time, and two times.} 
  \label{fig:timeml}
  \end{center}
\end{figure}

\subsubsection{TimeML}
Deficiencies in TIDES led to the development of TimeML~\citep{Sauri:06}, another markup language for annotating temporal information, originally designed for news. As shown in Table~\ref{tab:constraints}, TimeML contains all Allen's basic temporal relations. TimeML does not have Allen's \link{overlaps} and \link{overlapped\_by} relations directly, however, these relations can be expressed in the temporal representation of a TimeML document by combinations of other relations \citep{llorens2015semeval}. TimeML adds additional representational facilities. First, Allen's \link{equals} relation indicates two intervals have the same start and end time. In contrast, TimeML offers more refined relations: \link{simultaneous} represents two events that happened simultaneously, \link{during} represents an event persists throughout a temporal expression (duration), and \link{identity} represents event coreference relation between two events. Furthermore, TimeML includes relations for expressing sub-event structure (aspectual relations) and relations of conditional, hypothetical, or counterfactual nature (subordinating relations). Because the input to \TLEX is a TimeML graph, we give here a detailed exposition of the representational structure of TimeML.

TimeML uses TIMEX3 expressions, which introduce additional attributes and features to capture more complex temporal expressions accurately. These features include quantifiers ("quant"), frequencies ("freq"), anchoring for relative expressions ("anchorTimeID"), and comments for additional information. While TIMEX3 retains the three basic types of temporal expressions from TIMEX - dates (DATE), times (TIME), and durations (DURATION) - it also introduces a new type, "SET," to represent recurring events (e.g., {\it twice a week}). 

TimeML has three different types of links: temporal (\TLINK), subordinating (\SLINK), and aspectual (\ALINK). \TLINKs comprise $14$ different types of temporal relations between events and times. Figure~\ref{fig:timeml} graphically shows all types of \TLINKs in TimeML. The following example represents one of the \TLINKs, called \link{is\_included}. This \TLINK represents that the event ({\it went}) \link{is\_included} in ({\it on}) the time ({\it Monday}). The visual representation of this \TLINK and the other examples can be seen in Figure~\ref{fig:visual_representation}. $1$ is the first node of the link, $2$ is the second node of the link, and $X$ is the TimeML link type.

\begin{figure}[t]
\begin{center}
  \includegraphics[width=3cm]{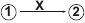}
  \caption{The visual representation of a generic link of TimeML: node $1$ and node $2$ represent two intervals, and $X$ represent the TimeML link type between the two intervals.}
  \label{fig:visual_representation}
\end{center}
\end{figure}

\ex. David \uline{\bf went}\mynum{1} to the school on \uline{\bf Monday}\mynum{2}. (\link{is\_included})

\ALINKs represent the relation between an aspectual event and its argument event. There are five types of \ALINKs (shown in Figure~\ref{fig:timeml}): \link{initiates}, \link{reinitiates}, \link{terminates}, \link{culminates}, and \link{continues}. The following example shows the aspectual relation between events, where the first event ({\it started}) \link{initiates} the second event ({\it study}). 

\ex. Mike \uline{\bf started}\mynum{1} to \uline{\bf study}\mynum{2}. (\link{initiates})

Note that some \ALINKs and \TLINKs may have the same temporal constraints (Table~\ref{tab:constraints}), but semantically they are different in TimeML. For example, \link{initiates} and \link{begins} have the same temporal constraints, however, \link{initiates} happen between an aspectual event and its argument event while \link{begins} does not require an aspectual event. 

\SLINKs are used for contexts introducing possible (modal), counterfactual, or conditional relations between two events, spanning six types: \link{modal}, \link{factive}, \link{counter\_factive}, \link{evidential}, \link{negative\_evidential}, and \link{conditional}. The \link{modal} relation introduces a relation to a possible event. In the following example, the event ({\it buy}) is merely a possibility and has not actually happened yet.

\ex. Cindy \uline{\bf promised}\mynum{1} him to \uline{\bf buy}\mynum{2} some nachos. (\link{modal})

The \link{factive} relation marks an entailment of an event's veracity. On the other hand, a \link{counter\_factive} relation marks a presupposition about the event's non-veracity. In other words, \link{factive} indicates a presupposition as to whether the event happened in the real world and \link{counter\_factive} indicates a presupposition as to whether the event did not happen in the real world. The following examples demonstrate the difference between these two SLINKs.

\ex. Katy \uline{\bf forgot}\mynum{1} that she \uline{\bf was}\mynum{2} in Miami last year. (\link{factive})

\ex. Katy \uline{\bf forgot}\mynum{1} to \uline{\bf buy}\mynum{2} some chocolate. (\link{counter\_factive}) \label{ex:cf} 

\link{evidential} relations are introduced by reporting events asserting that the argument event happened. Similarly, \link{negative\_evidential} relations are introduced by reporting events asserting that the argument event did not happen. The following examples show these two \SLINK types.

\ex. Colin \uline{\bf said}\mynum{1} he \uline{\bf went}\mynum{2} to the store. (\link{evidential}) 

\ex. Darius \uline{\bf denied}\mynum{1} that he \uline{\textbf{has coronavirus}}\mynum{2}. (\link{negative\_evidential})

Finally, the \link{conditional} relation identifies two events linked in a conditional manner, for example, with a signal such as ``if''.

\ex. If Amanda \uline{\bf marries}\mynum{1} him, she will \uline{\textbf{be happy}}\mynum{2}. (\link{conditional})

Language introducing an \SLINKs does not always entail a temporal relation, but if it does, a \TLINK representing that relation should be added to the graph as well. Furthermore, three types of \SLINKs explicitly indicate that the event is presumed to not have happened in the ``real world'' of the text. Therefore, ignoring subordinating relations or treating them as normal temporal relations gives an incorrect and impaired view of the temporal structure of the text.

Out of the three types of TimeML links, \ALINKs and \TLINKs represent temporal information in documents, while \SLINKs indicate non-temporal relations. On the other hand, \SLINKs and \ALINKs can only connect two events, while \TLINKs can relate two events, two times, or an event and a time.

A well-formed TimeML graph will have at most a single \TLINK and a single \ALINK relation between two nodes. This is a consequence of the TimeML rule, which says that between two event instances (or between an event instance and a time expression), there cannot be two or more \TLINK or \ALINK types. If there is both a \TLINK relation and an \ALINK relation, they must be temporally consistent~\citep{Pustejovsky:2003:timeml,Pustejovsky2003timebank}. As a consequence, TimeML graphs are atomic temporal graphs, hence a member of the tractable subclass of the interval algebra~\citep{Val86,Val87} (atomic graphs are called ``singleton networks'' in these papers). Another way to see this is by combining the fact that atomic temporal graphs can be equivalently transformed (in linear time) into their corresponding PA temporal graphs, and that consistency of PA temporal graphs can be checked in $O(n^2)$ time~\citep{Beek:92}.

\subsubsection{TimeML Annotations}
There are a number of manually and automatically annotated TimeML corpora. In this work, we tested our approach on four manually annotated TimeML corpora: TimeBank 1.2, the N2 corpus, the ProppLearner corpus, and the NewsReader MEANTIME corpus, all in English. The number of texts, words, events, temporal expressions, and relations are listed in Table~\ref{tab:corpora}. TimeBank 1.2 is drawn from various sources such as ABC, CNN, and the Wall Street Journal~\citep{Pustejovsky2003timebank}. The ProppLearner corpus was developed to enable the machine learning of Vladimir Propp's morphology of Russian hero tales and has 18 different layers of syntax and semantics annotated on it, including TimeML~\citep{Finlayson:2017}. Similarly, the N2 corpus is a collection of narratives relating to Islamic Extremism with 14 layers of annotation, including TimeML~\citep{Finlayson:2014}. Finally, the NewsReader MEANTIME corpus is a semantically annotated corpus of English Wikinews articles~\citep{Minard:16} (the corpus also includes TimeML annotations of article translations in Italian, Spanish and Dutch), covering four news topics: ``Airbus and Boeing'', ``Apple'', ``the Stock market'' and ``General Motors, Chrysler and Ford.''

Automatic TimeML annotation can be done via temporal information extraction systems for TimeML temporal expressions, events, and temporal relations~\citep{Verhagen:2005,Bethard:2013,Chambers:2014,Mirza:2016}, or via temporal dependency parsers~\citep{Zhang:2018}. While not perfect, these automatic TimeML annotators attain a reasonable, even good, level of performance under some conditions.  Because our goal in this paper is to present details of the \TLEX method for extracting timelines in an exact manner, we leave as future work the problem of applying our approach on automatically generated TimeML graphs. 

\begin{table*}[t]
\small
\centering
\begin{tabular}{lllllll} 
 \toprule
 \bf Corpus Name & \bf Texts & \bf Words & \bf Events & \bf TIMEXs & \bf Links & \bf Text Types \\
 \midrule 
 ProppLearner    & 15        & 18,862    & 3,438      & 142        & 2,778         & Russian Hero tales \\
 N2 Corpus       & 67        & 28,462    & 2,345      & 349        & 4,854         & Religious Texts, Magazine \\
 MEANTIME        & 120       & 13,981    & 2,096      & 525        & 1,717         & WikiNews Articles \\  
 TimeBank        & 183       & 68,555    & 7,935      & 1,414      & 9,615         & Newswire, Biography, etc. \\
 \midrule
 Total           & 385       & 129,860   & 15,814     & 2,430      & 18,964        &  \\
\end{tabular}
\caption{Summary of the corpora used in the experiments.}
\label{tab:corpora}
\end{table*}

\subsection{Prior Approaches to Timeline Extraction}
\label{sec:priorapproaches}

The goal of timeline extraction under {\it quantitative} (i.e., metric) frameworks is to produce timelines for temporal networks with metric temporal information for all times, events, and relations (i.e., numerical durations of nodes and edges in the temporal graph). Note that a solution for a quantitative framework does not immediately extend to a solution for qualitative cases such as Allen's algebra or TimeML. This is because qualitative frameworks explicitly allow the use of non-metric information in their temporal graphs; indeed, most of the temporal information in natural language is non-metrical.

In the field of qualitative spatial and temporal reasoning (QSTR), researchers have given significant attention to timeline extraction. \cite{Beek:92} presented an approach that transformed Allen's interval algebra temporal graphs into a point algebra network and solved them using backtracking techniques. \cite{Gerevini:95} built a ``timegraph'' given a set of point algebra relations and solved the timegraph using their own algorithm. Later, \cite{Wallgrun:06} presented the {\it SparQ (Spatial Reasoning done Qualitatively)} tool, which comprises a set of modules to provide different services for qualitative spatial reasoning, which is a similar problem to qualitative temporal reasoning. SparQ transforms a quantitative description of a spatial configuration into a qualitative description, applies the operations in the calculi to spatial relations, and finally performs computations on constraint networks. Similarly, \cite{Gantner:08} built the {\it Generic Qualitative Reasoner (GQR)} to solve binary qualitative constraint graphs, which takes a temporal calculus description and one or more constraint graphs as input and solves them using path consistency and backtracking. To check the consistency of large qualitative spatial networks, \cite{Sioutis:2014} presented {\it Sarissa}, which produces random scale-free-like qualitative spatial networks using the Barabási-Albert (BA) model and uses a hash table-based adjacency list to efficiently represent and reason with them. Finally, \cite{Kreutzmann:2014} showed how to use AND-OR linear programming (LP) and mixed integer linear programming (MILP) to solve qualitative graphs, a set that includes graphs represented in Allen's algebra.  

In the NLP field, researchers have explored different methods for timeline extraction in general, and from TimeML annotated texts in particular. A prominent example of general timeline extraction is the news timeline summarization (TLS) task. Starting from the seminal work of Allan and his collaborators~\citep{SA00,AGK01}, timeline summarization focuses on finding specific date mentions in texts and organizing the text or its subpart along a timeline indexed by date. TLS has attracted much attention in the past two decades; for the latest developments see~\citep{GI20} and references therein. While this is a useful task for many purposes, it is much different from (and much more granular than) the task we set ourselves here, which is organizing all the events and times mentioned in a text into a (non-metric) timeline.

With regard to approaches that seek to extract timelines from TimeML annotated texts, we can divide these methods into two categories: (1) ML-based approaches that start from raw text, or raw text annotated only with events and time expressions (i.e., no links); and (2) temporal reasoning on fully TimeML annotated document. 

For the first category, researchers build a model to parse documents, predict TLINKs, and infer the interval-based order. \cite{Mani:2006} demonstrated a machine learning method to annotate events and partially order them using a qualitative temporal graph generated using Allen's temporal relations. However, this model only predicted three relations---\link{before}, \link{after}, and \link{simultaneous}---and thus excludes large portions of TimeML, and only achieves roughly 75\% ordering accuracy. \cite{Do:2012} generated the same three relations to obtain a full ordering, achieving a similar accuracy of 73\%. In addition to the imperfect performance of these systems, in both cases the methods only consider intervals, rather than start and end points, and so lose much detailed temporal information. Finally, \cite{Kolomiyets:2012} presented two models for extracting timelines from qualitative temporal graphs, one based on shift-reduce parsing, and one based on graph parsing. Both approaches take a sequence of event words as input and produce a tree structure. Achieving 70\% accuracy, their model generated six temporal relations---\link{before}, \link{after}, \link{includes}, \link{is\_included}, \link{identity}, \link{overlap}---again, only a subset of the full possible set of temporal relations. Similarly, \cite{Zhang:2018} built a parser to extract dependency trees, which can easily be transformed into a timeline. They used LSTM in their pipeline and achieved 76\% accuracy. Later, \cite{ocal-2020} showed that temporal dependency trees omit temporal information. Using CNNs and LSTMs, \cite{leeuwenberg2020towards} demonstrated a deep learning model to predict events' start time-point, duration, and end time-point. Their model extracts three types of \TLINKs between events (\link{before}, \link{after}, and \link{overlaps}), achieving 77\% accuracy. And finally, \cite{mathur-etal-2022-doctime} built a temporal parser to extract a temporal dependency graph and maintain the temporal order of the events using contextual features such as BERT. They achieved 77\% accuracy on four \TLINKs---\link{before}, \link{after}, \link{includes}, and \link{overlaps}. As can be seen these ML-based approaches have only been applied to partial TimeML annotations (3--6/25 TimeML links) with imperfect performance (70--77\% accuracy).

In the second category---i.e., temporal reasoning-based methods---\cite{Tango} introduced Tango, a TimeML parser tool to parse the TimeML annotated documents and create a TimeML graph.  Using the TIMEX values, Tango displays the graph in a timeline form, where each section of the timeline contains a TIMEX and all the events connected to the TIMEX, however, it doesn't provide the global order of events. Although Tango checks the consistency of the TimeML graphs using path consistency, it didn't report any inconsistencies for their test corpus (the TimeBank corpus). Similar to Tango, \cite{TBOX} introduced TBOX, a TimeML parser. TBOX also generates a TimeML graph from a TimeML annotation, but it further removes the temporal closure links to display a simplified TimeML graph. TBOX displays each event in a box shape and places each box based on the temporal relation to present the timeline (e.g., if event $A$ is \link{before} event $B$, then box$_A$ would be on the left of box$_B$). However, this representation is problematic, considering temporal indeterminacy is already high in TimeML annotations~\citep{TBOX}. Although these two tools use all 14 types of \TLINKs, they ignore \SLINKs and \ALINKs. Therefore, many events would not be correctly displayed in the timeline structure.

Additionally, \cite{ning-etal-2018-multi} introduced a multi-axis annotation framework aimed at categorizing intention, opinion, hypothetical, generic, and negation events into distinct axes to distinguish them from static and recurring events. This innovative annotation approach led to the development of the MultiAxis Temporal Relations for Start-points (MATRES) corpus.

Summarizing these approaches for timeline extraction, we see that on the one hand, QSTR approaches transform Allen's interval algebra graphs into constraint graphs and solve them using constraint satisfaction techniques. Although these approaches provide a solution for solving qualitative temporal graphs, their methods cannot be applied directly to TimeML graphs because of subordinating relations, and they also do not detect or represent indeterminacy, nor do they help with correcting inconsistencies. ML-based approaches, on the other hand, have been limited by the number of relations considered and provide inexact solutions, presumably because of the noise inherent in a statistical solution. Finally, the TimeML parsers do not provide global orders for all events since they ignore aspectual and subordinating events.

\TLEX addresses both of these sets of deficiencies by drawing on both QSTR and NLP to provide a technique for extracting {\it exact} timelines from {\it full} TimeML graphs. Unlike prior approaches, \TLEX provides a timeline extraction method using {\it all} temporal relations, including aspectual and subordinating relations. It also exposes indeterminacy in ordering, which is a natural consequence of the ambiguity of natural language, and also clearly identifies the sources of inconsistency to enable manual correction. Expanding \cite{ning-etal-2018-multi}'s comprehensive multi-axis annotation framework, \TLEX creates a trunk-and-branch timeline structure that also effectively segregates counter-factive and negative evidential events from real-life events.

\section{TLEX Algorithms}\label{sec:approach}

\subsection{Overview of our approach}

\subsubsection{Problem statement}

Recall that our task is to, by taking all possible temporal relations (including aspectual and subordination relations) into consideration, extract from a TimeML annotated text a complete set of timelines that are consistent with the TimeML temporal graph, identify which portions of the timelines are indeterminate, and organize those timelines into a trunk-and-branch structure.


Without loss of generality, we may view the input TimeML temporal graph to {\TLEX} as a connected\footnote{Because if the TimeML temporal graph is not connected, then we can simply run {\TLEX} on each of its connected components separately.}, \emph{directed} graph $G_{TM}=(V_{TM}, E_{TM})$, where the vertex set $V_{TM}$ are events and time expressions, and there is a directed edge from vertex $u$ to vertex $v$ if there is a TimeML link from $u$ to $v$ (note that there can be more than one link, hence multiple relations corresponding to a single edge of $G$). That is, $E_{TM}=\{(u, v): u, v\in V_{TM} \text{ and there is a \TLINK, \SLINK, or \ALINK from $u$ to $v$}\}$.\footnote{A simple way to implement this would be to assign a $10$-bit vector for each edge as follows. First use $\bf{0}$ vector to denote no relation exists for each of the three types of links. The type of {\TLINK} relation of an edge can thus be represented with $\lceil \log_{2}(14+1)\rceil=4$ bits; similarly,  {\SLINK} relation needs $\lceil \log_{2}(5+1)\rceil=3$ bits and  {\ALINK} relation needs $\lceil \log_{2}(6+1)\rceil=3$ bits.} Let $n=|V_{TM}|$ be the number of intervals in the input TimeML temporal graph and $m=|E_{TM}|$ be the number of edges of $G_{TM}$. As $G_{TM}$ is connected, $m\geq n-1$; and since there can be at most one link of each type between an (ordered) pair of temporal intervals, the number of links in the TimeML temporal graph is at most $3m$.

\subsubsection{Structure of the algorithm}

\TLEX consists of the following five steps:

\begin{enumerate}

\item \textsc{Partitioning}, in which the input TimeML temporal graph is partitioned into subgraphs connected with only temporal and aspectual links;

\item \textsc{Transforming}, in which each TimeML temporal subgraph is transformed into a PA constraint subgraph with time-points as vertices and basic temporal relations ($<$, $=$) as edge labels;

\item \textsc{Consistency Checking}, in which the consistency of each PA constraint subgraph is checked, a maximal list of inconsistent cycles is output for manual correction and only consistent PA constraint subgraphs are passed on to the next step;

\item \textsc{Timeline Generation}, in which the \emph{minimum} timeline for all the time-points in each PA constraint subgraph is generated;

\item \textsc{Indeterminacy Detection}, in which each pair of time-points in the generated timeline can be checked whether there is any ordering indeterminacy\footnote{More precisely, by \emph{indeterminacy}, we mean the possibility of generating some other timeline in which the ordering of the two time-points is reversed.} between them.
\end{enumerate} 

\subsubsection{A running example}

In the following sections we provide detailed descriptions and analysis of each step, using the following text as a running example: 

\ex.
\label{ex:running}
After \uline{working}\mynum{1} in the \uline{morning}\mynum{2} John \uline{came back}\mynum{3} home. He \uline{wished}\mynum{4} that he \uline{didn't go}\mynum{5} to work, so he could \uline{be happy}\mynum{6}. Later, his phone \uline{rang}\mynum{7} and he \uline{answered}\mynum{8} it. As soon as he \uline{picked up}\mynum{9}, Paul \uline{started}\mynum{10} \uline{complaining}\mynum{11}. He was quickly \uline{bored}\mynum{12}, but \uline{realized}\mynum{13} that if he \uline{hung up}\mynum{14}, Paul \uline{would be mad}\mynum{15} for the entire \uline{weekend}\mynum{16}. So he \uline{continued}\mynum{17} \uline{to listen}\mynum{18}. Then he \uline{remembered}\mynum{19} he \uline{ignored}\mynum{20} Paul's text in the \uline{morning}\mynum{2}.

In this example, each event is underlined and given a numerical subscript for reference. One thing notable about this example text is that events 14 and 15 did not happen in the ``real world'' of the story: because John did not \uline{hang up}\mynum{14} the phone, Paul didn't \uline{get mad}\mynum{15}. These two events are related to the ``real world'' timeline by subordinating links and should be isolated onto their own ``subordinated'' timeline.

We informally define a ``main'' or ``real world'' timeline as a timeline that contains the events and times that actually happen in the ``world'' described in the text. Similarly, we informally define ``subordinated'' timelines as containing events that did not occur in the world described in the text. As noted, in the example text, the events \uline{hang up}\mynum{14} and \uline{would be mad}\mynum{15} belong to a subordinated timeline. It is possible that there are multiple main timelines in a text; we treat this possibility formally in \S\ref{sec:maintimelines}.

\subsubsection{Depth First Search (DFS)}
We will need the standard DFS (depth first search) algorithm to explore a directed graph and detect cycles in the graph. For completeness, we provide the notation and the pseudo-code of the DFS version presented in Section 20.3 of \cite{CLRS22}, omitting the $time$ field as we do not use it.

Specifically, on an input directed graph $G=(V, E)$ stored as an adjacency list, DFS initially colors all vertices as \textsc{white}. When DFS begins to explore the neighborhood of a vertex $u$ (by recursively calling \textsc{DFS-Visit} on each vertex $v$ which is in $u$'s adjacency list $G.Adj[u]$ and whose color is still \textsc{white}), it changes $u$'s color to \textsc{gray}. When DFS finishes examining $u$'s neighbors, it changes $u$'s color to \textsc{black} and the color of $u$ will never be changed again after that. Apart from color, each vertex also has a field $\pi$, which is used to store its parent node in the DFS forest constructed in the process of graph exploration. It is known that DFS runs in time linear time, namely $O(n+m)$ where $n=|V|$ is the number of vertices and $m=|E|$ is the number of edges of the graph $G$.
 
\begin{algorithm}[ht]
\caption{\textsc{DFS}$(G)$} \label{alg:DFS}
\begin{algorithmic}[1]
\small
    \Require directed graph $G=(V, E)$
    \Ensure a DFS forest
    \ForAll{$v \in V(G)$} 
    	\State $u.color = \textsc{white}$
    	\State $u.\pi = \textsc{nil}$
    \EndFor
    
    \ForAll{$v \in V(G)$}
    	\If {$u.color == \textsc{white}$}
    		\State \textsc{DFS-Visit}$(G, u)$
	\EndIf
    \EndFor
    
\vspace{1em}
\hspace{-4em}{\bf\textsc{DFS-Visit}$(G, u)$}
\vspace{0.5em}
\State $u.color= \textsc{Gray}$
\ForAll{$v \in G.Adj[u]$}
	\If {$v.color == \textsc{white}$}
    		\State $v.\pi = u$
    		\State  \textsc{DFS-Visit}$(G, v)$
	\EndIf
\EndFor
\State $u.color = \textsc{black}$
\end{algorithmic}
\end{algorithm}

\subsection{Partitioning}\label{sec:partitioning}

In this step, \TLEX partitions the TimeML graph into subgraphs that are connected only with temporal and aspectual relations. The reason is that temporal and aspectual relations represent temporal information about the text, while subordinating relations do not. 

To do this, we first transform the input graph $G_{TM}=(V_{TM}, E_{TM})$ into an \emph{undirected} graph $G'_{TM}=(V_{TM}, E'_{TM})$ by keeping the same vertex set and adding an (undirected) edge $(u,v)$ between two vertices $u$ and $v$ if there is a temporal or aspectual relation from $u$ to $v$, or a temporal or aspectual relation from $v$ to $u$. That is, we ignore all \SLINK relations and view \TLINK and \ALINK relations as undirected edges. This can clearly be done in $O(n+m)$ time. Next, we run DFS as illustrated in Algorithm~\ref{alg:DFS} on $G'_{TM}=(V_{TM}, E'_{TM})$ and partition it into connected subgraphs. Since the transformed graph $G'_{TM}=(V_{TM}, E'_{TM})$ has $|E'_{TM}|\leq |E'_{TM}|$, the DFS runs in $O(n+m)$ and hence the total running time of \textsc{Partitioning} is also $O(n+m)$. We denote the resulting subgraphs as $G^{(1)}_{TM}, \ldots, G^{(k)}_{TM}$.

Figure~\ref{fig:example-timeml-graph} illustrates \textsc{Partitioning} on the TimeML graph corresponding to Example~\ref{ex:running}. Subgraphs are separated with dashed lines. Since there are no temporal or aspectual links between event pairs 5 \& 6 and 14, 15 \& 16 with the rest of the graph (but the graph is otherwise connected), \TLEX partitions the graph into three temporally connected subgraphs: the main, or ``real world'' subgraph (1-2-3-4-7-8-9-10-11-12-13-17-18-19-20) and the subordinated subgraphs (5-6 \& 14-15-16). \TLEX will subsequently extract a separate timeline from each of these subgraphs.

After the partitioning, \TLEX stores the \emph{connecting points} between the main subgraphs and subordinated subgraphs (these points can be easily found by scanning the subordinating links: if a subordinating link $(u, v)$ connecting one vertex from a main subgraph and one vertex from a subordinated subgraph, then both vertices $u$ and $v$ are connecting points). Connecting points will be used later to build the trunk-and-branch timeline. The connecting points in Example~\ref{ex:running} are 4 \& 5 and 13 \& 14.

\begin{figure}[ht] 
\begin{center}
  \includegraphics[width=\textwidth]{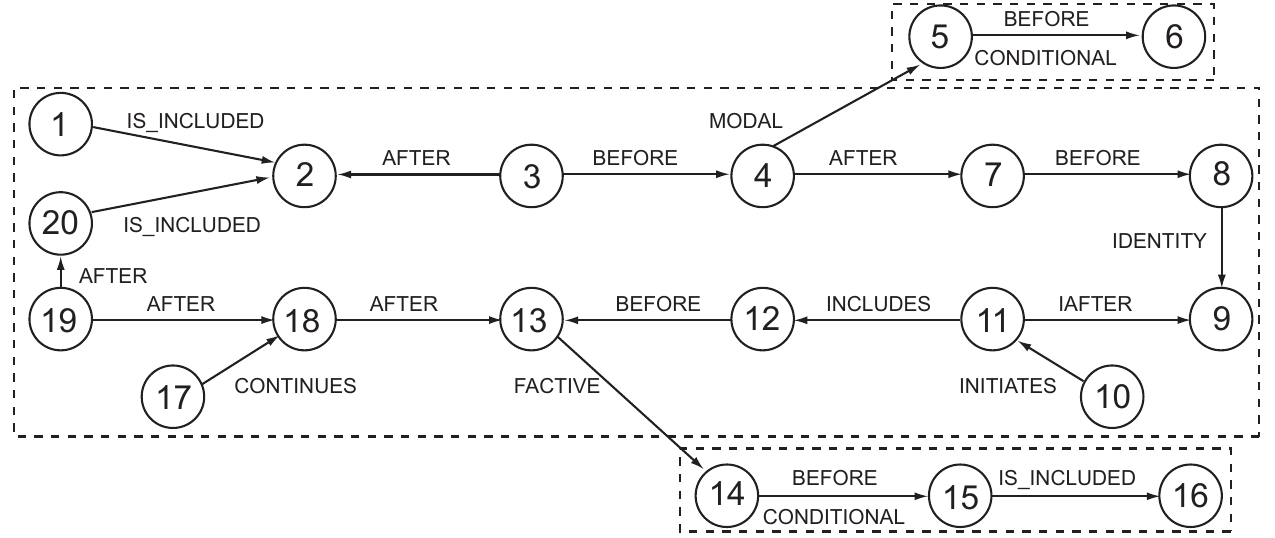}
  \caption{Visualization of the TimeML graph from the example. Numbers correspond to the events in the text, and arrows correspond to the temporal,  aspectual, or subordinating links. The two temporally and aspectually connected subgraphs are separated by dashed lines,  and links on the ``real world'' timeline are bolded.}
  \label{fig:example-timeml-graph}
\end{center}
\end{figure}

\subsection{Transforming} \label{sec:transforming}


In TimeML,  DATE, DURATION, and SET temporal expressions are regarded as temporal intervals, with starting and ending time points.\footnote{Note that \link{reinitiates} occurs between a continuous event and a reinitiating trigger word (such as \textbf{resuming} or \textbf{renewing}). If the event is not continuous, then there will be two different event instances instead of one, and the second event instance will be initiated by the trigger word.} TIME temporal expressions represent a specific moment in time, such as {\it 9:46 pm} or {\it 1 am}. Unlike intervals, they don't have a duration and simply pinpoint a single time-point in PA \citep{Bartak:2014}. We transform each temporal \emph{intervals} and time stamps of TimeML subgraph $G^{(i)}_{TM}$  into a point algebra (PA) constraint graph $G^{(i)}_{PA}=(V^{(i)}_{PA}, E^{(i)}_{PA})$, for every $i=1,\ldots, k$.

Let $G^{(i)}_{TM}=(V^{(i)}_{TM}, E^{(i)}_{TM})$ be any of the interval TimeML subgraphs. Note that we still view $G^{(i)}_{TM}$ as a \emph{directed} graph following the same definition as the original graph $G_{TM}$; that is, there is a directed edge from vertex $u$ to vertex $v$ in $G^{(i)}_{TM}$ if and only if there is a \TLINK relation or an \ALINK relation between them.

Recall that each vertex $v\in V^{(i)}$ actually represents a temporal interval $I_v$, so we first replace each interval $v$ with two time-points, the start time-point $t^-_v$ and the end time-point $t^+_v$ of the interval $I_v$. More formally, we let 
\[
V^{(i)}_{PA}=\{t^-_v: v\in V^{(i)}_{TM}\} \cup \{t^+_v: v\in V^{(i)}_{TM}\}.
\]

Now we construct the edges of $G^{(i)}_{PA}$ (i.e., basic PA constraints) as follows. First, we add the constraint $t^-_v < t^+_v$ for every $v\in V^{(i)}$. Using the fact that each \TLINK or \ALINK can be represented as a simple conjunction of the basic temporal constraints {\it less than} (\textless) and {\it equals} (=) as shown in Table~\ref{tab:constraints}, we next, for each directed edge in $E^{(i)}_{TM}$, add its corresponding basic PA constraints to $E^{(i)}_{PA}$. More formally, if we use $\phi(u, v, t^{+}_{u}, t^{-}_{u}, t^{+}_{v}, t^{-}_{v})$ to denote the set of PA constraints, as specified in Table~\ref{tab:constraints}, between $\{t^{+}_{u}, t^{-}_{u}\}$ and $\{t^{+}_{v}, t^{-}_{v}\}$ which are imposed by the TimeML constraint of an edge $(u, v)$ in $G^{(i)}_{PA}$, then
\[
E^{(i)}_{PA}= \left(\cup_{v \in V^{(i)}_{TM}} \{t^-_v < t^+_v\} \right) \bigcup 
\left( \cup_{(u, v) \in E^{(i)}_{TM}} \phi(u, v, t^{+}_{u}, t^{-}_{u}, t^{+}_{v}, t^{-}_{v}) \right).
\]

It is easy to check that the blowup of PA constraint graph size is at most a constant factor, namely $|V^{(i)}_{PA}|=2|V^{(i)}_{TM}|$ and $|E^{(i)}_{PA}| \leq |V^{(i)}_{TM}| + 2|E^{(i)}_{TM}|$. In the following we write 
\[
n_i :=|V^{(i)}_{PA}|
\]
for the number of vertices in the $i^{\text{th}}$ PA constraint subgraph and 
\[
m_i := |E^{(i)}_{PA}|
\]
for the number of edges in the $i^{\text{th}}$ PA constraint subgraph. Note that $\sum_{i=1}^{k} n_i =2n$ and $\sum_{i=1}^{k} m_i \leq n+2m$. It follows that \textsc{Transforming} takes linear time $O(n+m)$ to compute. While certain temporal details like verb tense and specific linguistic signals are lost during the transformation of the TimeML graph to the PA graph, the information pertaining to the global order of events and times remains intact. 

Figure~\ref{fig:example-constraint-graph} illustrates the transformed PA constraint graph from the TimeML temporal graph shown in Figure~\ref{fig:example-timeml-graph}. 

\begin{figure}[ht] 
\begin{center}
  \includegraphics[width=\columnwidth]{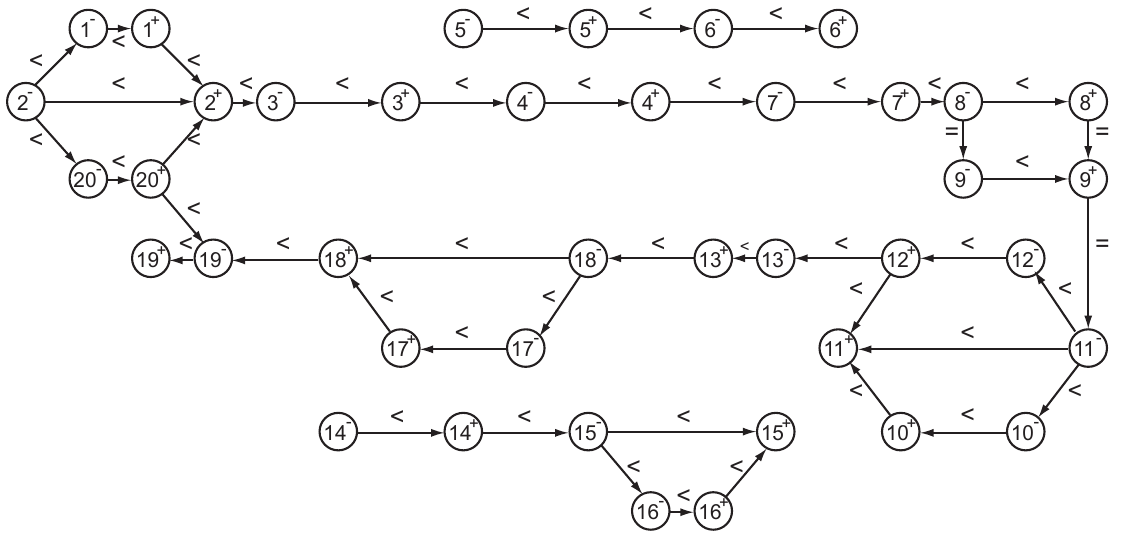}
  \caption{The three PA constraint subgraphs transformed from the three temporally and aspectually connected subgraphs in Figure~\ref{fig:example-timeml-graph}. These are produced by replacing each node $I$ with its start and end time-points $I^-$ and $I^+$,  and replacing each temporal or aspectual link with the set of basic temporal relations shown in Table~\ref{tab:constraints}.}
  \label{fig:example-constraint-graph}
\end{center}
\end{figure}

\subsection{Consistency Checking}\label{sec:correcting}

A \emph{timeline} for a PA constraint graph $G_{PA}=(V_{PA},E_{PA})$ is a function $L$ that maps each time-point in $V_{PA}$ to a (not necessarily unique) point on the real axis such that all time-point constraints imposed by $E_{PA}$ are satisfied. A PA constraint graph $G_{PA}=(V_{PA},E_{PA})$ is \emph{consistent} if there exists a timeline $L$ for $G$. Therefore we need to check the consistency of each PA constraint subgraph before passing it on for timeline generation.

It was shown by \cite{Beek:92} that a PA constraint graph is consistent if and only if it does not have any of the following three types of inconsistent cycles.

\begin{definition}[Inconsistent Cycle]
\label{def:inconsistentcycle}
An inconsistent cycle is a cycle in a PA constraint graph with one of the following types (for simplicity of notation, we use ``$\neq$'' as shorthand for the relation ``$<$'' or ``$>$'' and use ``$\leq$'' as shorthand for the relation ``$=$'' or ``$<$'') :
\begin{enumerate}[label=(\roman*)]
    \item $v = ... = w \neq v$;
    \item $v \leq ... \leq w \leq ... \leq v \neq w$;
    \item $v < ... < w < ... < v$, where all but one of the $<$ can be $\leq$ or $=$ (Note that this case also covers the case of self-loops, $v < v$).
\end{enumerate}
\end{definition}

An example of an inconsistent cycle is shown in Figure~\ref{fig:inconsistent_cycle}.

\begin{figure}[hb]
\begin{center}
  \includegraphics[width=0.3\textwidth]{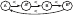}
  \caption{Example of an inconsistent cycle generated by A \link{before} B and B \link{before} A.} 
  \label{fig:inconsistent_cycle}
\end{center}
\end{figure}

In the third step of  \TLEX, namely \textsc{Consistency Checking}, given a PA constraint subgraph $G^{(i)}_{PA}=(V^{(i)}_{PA}, E^{(i)}_{PA})$, our task is to check if $G^{(i)}_{PA}$ is consistent; moreover if $G^{(i)}_{PA}$ is inconsistent, output a \emph{Maximal List of Inconsistent Cycles} (MLIC) which includes the set of all inconsistent cycles detected in our checking procedure. This list of TimeML links that correspond to these points can then be provided to human annotators for manual corrections. Although there has been some research in the field of qualitative spatial reasoning focused on automatically correcting inconsistencies in point algebra constraint graphs in a minimum or approximate manner~\citep{Iwata:2013,Condotta:2016}, these approaches cannot be applied to TimeML graphs. Inconsistencies in TimeML graphs usually come from incorrect or incomplete annotations; more rarely, there may be fundamental inconsistencies in the meaning of the language. Therefore, errors need to be fixed by reference to the original text. It is possible that one could develop other knowledge- or inference-based methods for automatically correcting these inconsistencies, but we will leave the exploration of this for future work. This means that any inconsistencies discovered by \TLEX need to be manually corrected before proceeding. We discuss more specific cases of inconsistencies in selected corpora in Section~\ref{sec:evaluation}.

\textsc{Consistency Checking} comprises the following three parts.

\subsubsection{Merging equal time-points into ``compound'' time-points}

In this part, we merge time-points that are inter-connected by {\it equals} (=) relations. Specifically, we keep only the $=$ relations in the PA constraint subgraph and view it as an \emph{undirected} graph. We then run DFS on the modified graph to partition the vertices in $V^{(i)}_{PA}$ into connected components $V_1, \ldots, V_{\ell}$ (Note that some $V_i$ may contain a single vertex; indeed that should be the case for most $V_i$'s when the PA constraint graph is generated from a real-world text corpus). Time-points in each $V_i$ are inter-connected by $=$ relations and hence they should be equal to each other. We call each such $V_i$ a ``compound'' time-point. Since running DFS on (a subgraph of) each PA constraint subgraph takes time at most linear in its size, the total running time for processing all PA constraint subgraphs takes $O(m+n)$ time.

In the example, PA constraint graph in Figure~\ref{fig:example-constraint-graph}, $8^-$ and $9^-$ will be merged into one ``compound'' time-point, $8^+, 9^+$ and $11^-$ will be merged into another ``compound'' time-point, and the remaining time-points will be singleton ``compound'' time-points.

\subsubsection{Building ``compound'' PA constraint graph, detecting type (i) and type (ii) inconsistency}

We now build a ``compound'' PA constraint graph $\tilde{G}^{(i)}_{PA}=(\tilde{V}^{(i)}_{PA}, \tilde{E}^{(i)}_{PA})$: the vertex set $\tilde{V}^{(i)}_{PA}$ is just the set of ``compound'' time-points created in the previous part, and the labels of the edges will be of {\it less than} (\textless) relation only and can be generated by scanning the set of {\it less than} (\textless) labeled edges in  $E^{(i)}_{PA}$. More specifically, we first associate a new field ``compound index'' (an integer between $1$ and $\ell$) to each vertex in $V^{(i)}_{PA}$, indicating to which ``compound' time-point the corresponding time point belongs. We then create an $\ell \times \ell$ Boolean-valued table $T$ whose $(i, j)$-entry stores whether ``compound'' time-point $i$ is {\it less than} ``compound'' time-point $j$. We now go through all {\it less than} (\textless) labeled edges in  $E^{(i)}_{PA}$ and update the entries of $T$ accordingly. During this process, we also detect inconsistency cycles as follows:

\begin{description}

\item [detecting type (i) inconsistency:]
If there is a constraint $u<v$ and $u, v$ both belong to the same ``compound'' time-point, then a type (i) inconsistency cycle is found, and we add all time-points in this ``compound'' time-point to the MLIC;

\item [detecting type (ii) inconsistency:]
For every constraint $u<v$ where $u$ belongs to ``compound'' time-point $i$ and $v$ belongs to ``compound'' time-point $j$, we update the $(i, j)$-entry of $T$ to \textsc{TRUE}. We then check the $(j, i)$-entry of $T$ and a type (ii) inconsistency cycle is detected if that entry is also \textsc{TRUE}. If so, we add all time-points in both ``compound'' time-point $i$ and ``compound'' time-point $j$ to the MLIC.

\end{description}

Note that now the ``compound'' PA constraint graph has $\ell\leq n_i$ vertices, at most $m_i$ directed edges and all edges are label with  {\it less than} (\textless) relations. As we need to build the table $T$, the running time is $O(n^2+m)$.

\subsubsection{Cycle-finding in ``compound'' PA constraint graph---detecting type (iii) inconsistency}

To detect type (iii) inconsistency in $\tilde{G}^{(i)}_{PA}=(\tilde{V}^{(i)}_{PA}, \tilde{E}^{(i)}_{PA})$, we apply the well-known cycle-detection method of DFS. Run the DFS described in Algorithm~\ref{alg:DFS} on $\tilde{G}^{(i)}_{PA}$ with the following slight modifications: if when vertex $u$ explores its neighboring vertices in Line 9 and finds that $v$'s color is \textsc{gray}, then a cycle is detected. Proof of the correctness of this method is well-known, but for completeness we provide a sketch here. Let $v_1 < v_2 < \cdots < v_p < v_1$ be a cycle in $\tilde{G}^{(i)}_{PA}$ and without loss of generality $v_1$ is the first vertex to be discovered by DFS. By the White-path Theorem,~\citep{CLRS22} Theorem 20.9, $v_p$ is a descendant of $v_1$ in the DFS forest generated by the algorithm; it follows that when $v_p$ is discovered in Line 12, the color of $v_1$ is still \textsc{gray}, hence the cycle will surely be detected. On the other hand, if a vertex $u$ finds that $v$'s color is \textsc{gray} when exploring its neighbors, then $v$ is an ancestor of $u$, meaning that there is a directed path from $v$ to $u$. But $v$ is a neighbor of $u$ implies there is an edge from $u$ to $v$. Therefore there must be a cycle containing both $u$ and $v$.

To detect all possible type (iii) inconsistency cycles, each time a cycle is detected we remove edge $(u,v)$ from $\tilde{G}^{(i)}_{PA}$ to break the cycle, add the cycle to the MLIC (the cycle can be found by tracing back recursively the parent vertex $u.\pi$ starting from $u$ until we reach vertex $v$), and keep running the DFS algorithm till it finishes. The running time of this part is $O(m+n)$

\paragraph{Summary of \textsc{Consistency Checking}.}

From our analysis in each of the three parts, the overall running time of \textsc{Consistency Checking} is clearly $O(m+n^2)$. Note that, if the algorithm does not detect any inconsistency in the ``compound'' PA constraint subgraph $\tilde{G}^{(i)}_{PA}$ and passes it on to next step, then the ``compound'' PA constraint subgraph satisfies

\begin{enumerate}

\item $\tilde{G}^{(i)}_{PA}$ is a directed acyclic graph (DAG);

\item All the constraints are {\it less than} (\textless) relations; that is, if $(u, v)$ is an edge in $\tilde{G}^{(i)}_{PA}$ then for any timeline $L$ satisfying $\tilde{G}^{(i)}_{PA}$, $L(u) < L(v)$.

\end{enumerate}

\subsection{Timeline Generation}
\label{sec:solving}

\begin{figure}
\begin{center}
  \includegraphics[width=0.5\textwidth]{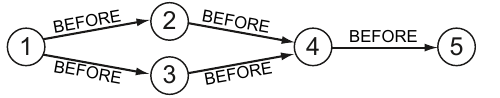}
  \caption{A TimeML graph with an indeterminacy. While the order between interval $I_4$ and interval $I_5$ is fixed, 
  the relative order between interval $I_2$ and interval $I_3$ is indeterminate.}\label{fig:old_example-indeterminate-graph}
\end{center}
\end{figure}

\subsubsection{Normal Form Timeline}

The two properties listed at the end of last section of the ``compound'' PA constraint subgraph allow ordering \emph{linearly}, namely topological sort, all the time-points involved. Recall that a \emph{topological sort} of a directed acyclic graph $G=(V,E)$ is a linear ordering of all its vertices such that if $G$ contains an edge $(u, v)$, then $u$ appears before $v$ in the ordering. However, our goal in \textsc{Timeline Generation} is slightly different: for every ``compound'' PA constraint subgraph $\tilde{G}^{(i)}_{PA}$, we want to generate a \emph{shortest} timeline for time-points in $\tilde{G}^{(i)}_{PA}$, for which we provide a more formal definition now. We prioritize generating the shortest timeline among potentially multiple solutions to establish a reference point for subsequent indeterminacy calculations. This allows for an equitable distribution of indeterminacy scores across all timeline segments, as detailed in \S~\ref{sec:indeterminacy}. To keep notation simple, in the rest of this section we use a general PA constraint graph $G=(V,E)$ instead of $\tilde{G}^{(i)}_{PA}$ as the input to \textsc{Timeline Generation}.

First, in order to make notation simple and eliminate unnecessary ambiguities, we require the values of a timeline to be positive integers, hence the notion of \emph{normal form timelines}. Let $\N=\{1, 2, \ldots\}$ denote the set of natural numbers.

\begin{definition}[Timeline and normal form timeline]
Let $G=(V,E)$ be a PA constraint graph. A function $L: V\to \R$, which maps every time-point in $V$ to a real number, is called a \emph{timeline} of $V(G)$ if all the temporal constraints imposed by $E$ are satisfied. A timeline is said to be in \emph{normal form} if the range of the timeline $L$ is further restricted to be $\N$. The \emph{length} of a normal form timeline is the maximum value of the timeline $\ell(L)=\max_{u\in V} L(u)$. A timeline $L$ is called a \emph{minimal normal form timeline} of $G$ if for any normal form timeline $L'$ of $V(G)$, $\ell(L) \leq \ell(L')$. Finally, a timeline $L$ is called a \emph{minimum normal form timeline} if for any normal form timeline $L'$ of $V(G)$ and every vertex $u$ in $V$,  $L(u) \leq L'(u)$.
\end{definition}

It follows immediately from the definitions that if a timeline $L$ is a minimum normal form timeline then it is also a minimal normal form timeline. However, the converse is in general not true. The motivation behind our definitions is the following. It is possible that there are multiple normal form timelines of a PA constraint graph that all have the minimum possible length. In order to eliminate such ambiguity, we further require the timeline should map every time-point to its minimum possible value in any consistent timeline. Indeed, one can show that for any PA constraint graph $G$, the minimum normal form timeline for $G$ is unique\footnote{The uniqueness follows by inspecting the Greedy Kahn's algorithm (Algorithm~\ref{alg:Kahn}) and proving the unique optimality of its output timeline by an induction on time-points that are added to set $S$ at iteration $t$.} and hence it is justified to call a timeline \emph{the} minimum normal form timeline for $G$.

Consider the example shown in Figure~\ref{fig:old_example-indeterminate-graph}, one possible normal form timeline (after transforming the original Allen algebra on temporal intervals into PA on time-points) $L_1$ is to map interval $I_1$ to $[1, 2]$, interval $I_2$ to $[3, 4]$, interval $I_3$ to $[5, 6]$, interval $I_4$ to $[7, 8]$, and interval $I_5$ to $[9, 10]$, with $\ell(L_1)=10$. Another normal form timeline $L_2$, as shown in Figure~\ref{fig:old_example-indeterminate-timeline}, maps interval $I_1$ to $[1, 2]$, both interval $I_2$ and interval $I_3$ to $[3, 4]$, interval $I_4$ to $[5, 6]$, and interval $I_5$ to $[7, 8]$, with $\ell(L_2)=8$. It is not hard to be convinced that $L_2$ is the minimum normal form timeline for the original TimeML graph.

\begin{figure}[h]
\begin{center}
  \includegraphics[width=0.6\columnwidth]{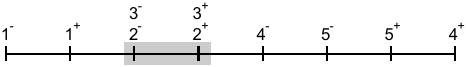}
  \caption{The minimum timeline extracted from the TimeML graph in Figure~\ref{fig:old_example-indeterminate-graph}, with the indeterminate section highlighted.}
  \label{fig:old_example-indeterminate-timeline}
\end{center}
\end{figure} 

\TLEX assumes, in the absence of additional information, that the minimum normal form timeline is the best.

\subsubsection{Minimum normal form timeline generation algorithm}

\paragraph{Kahn's topological sort algorithm.}

We first review the classic Kahn's algorithm~\citep{Kah62} for topological sort as it is the basis of our minimum normal form timeline generation algorithm. It is helpful to view topological sorting vertices in $V$ as a scheduling problem (for example, taking courses to complete an undergraduate degree), and think edges in $E$ as scheduling constraints (if there is an edge from $u$ to $v$, it means course $u$ is a prerequisite for course $v$). Kahn's topological sort algorithm for an acyclic directed graph (DAG) $G=(V,E)$ is to repeatedly find a vertex of in-degree $0$, output it, and remove it and all of its outgoing edges from the graph. The key observation of the algorithm is that, as long as the in-degree of a vertex becomes $0$, it is free, and we can safely schedule it. Once it is scheduled, the constraints it imposes on other vertices can be removed, and this potentially frees more vertices to schedule. The correctness of the algorithm follows from the fact that there always exists some vertex of in-degree $0$ in a DAG.\footnote{To see this, suppose this is not the case. Then every vertex has at least one incoming edge. Now starting from any vertex $u$ of $G$, repeatedly trace along (in the opposite direction) any of its incoming edges. Since every vertex has at least one incoming edge, this process will keep running. But graph $G$ is finite, therefore the path of this process must eventually visit some vertex it has visited before. That is, there must exist a cycle in this path. Reversing the direction of each edge in this path gives rise to a cycle in the original graph $G$, thus a contradiction.} Note that removing a vertex of in-degree $0$ and deleting all its outgoing edges from a DAG can never create new cycles; therefore, the algorithm will proceed until all vertices are put in a topological ordering. The running time of Kahn's algorithm is $O(m+n)$ by the following analysis. First, we maintain a queue which is set to be empty initially for the vertices of in-degree $0$. Second, in the preprocessing stage, by scanning all edges in $E$, we can compute the in-degree of every vertex, and place all in-degree $0$ vertices in the queue; this takes $O(m+n)$ time. Last, after removing a vertex $u$ of in-degree $0$ from the queue, we go through $u$'s outgoing edges, delete them from $E$, and update the in-degrees of $u$'s out-neighbors, and insert any vertex into the queue if its in-degree becomes $0$. Note that we need to process each edge only once
and process each vertex $O(1)$ times in the whole algorithm, so this part also takes $O(m+n)$ time.

\paragraph{Greedy Kahn's algorithm for minimum normal timeline}

Note that the Kahn's algorithm described above already provides a normal form timeline for time-points in $V$.  What we do next is to slightly modify the algorithm so that it outputs the minimum normal form timeline.

Since a normal form timeline can only schedule time-points at $t=1, 2, \ldots$, we partition our timeline construction into stages accordingly: at stage $i$, the algorithm determines which time-points the timeline $L$ should map to value $i$. A simple idea to construct a \emph{minimum} normal form time is, therefore to \emph{greedily} schedule as many time-points as allowed by the PA constraint graph. What is the maximal set of time-points that can be scheduled first, i.e. mapped to value $1$? Clearly, since all time-points with in-degree $0$ do not have any constraint on them, thus can all be mapped to value $1$. Following the idea of Kahn's algorithm, after mapping these time-points to $1$, we delete their out-going edges (this will create some new time-points of in-degree $0$ as argued in the preceding paragraph) and map all new in-degree $0$ time-points to value $2$. We keep repeating this process until all time-points have been scheduled.

\begin{algorithm}[ht]
\caption{\textsc{Greedy Kahn's Algorithm}} \label{alg:Kahn}
\begin{algorithmic}[1]
\small
    \Require A consistent PA constraint graph $G=(V, E)$
    \Ensure A minimum normal form timeline $L: V \to \N$
    \State $t = 0$
    \ForAll{$v \in V(G)$} 
    	\State $d_{in}(v)=0$
    \EndFor
    \State $S=\emptyset$
    \ForAll{$(u,v) \in E(G)$} 
    	\State $d_{in}(v)=d_{in}(v)+1$ \Comment{compute the in-degree of every time-point}
    \EndFor
    \While{$V \neq \emptyset$}
    	\State $t=t+1$
    	\State $S=\{u\in V: d_{in}(u)=0\}$ \Comment{set $S$ contains all currently ``free'' time-points}
    	\ForAll{$u \in S$}
    		\State remove vertex $u$ from $V$ 
    		\State $L(u)=t$ \Comment{greedily schedule all time-points that are ``free''}
    		\State delete all outgoing edges of $u$ \Comment{remove all constraints $u$ imposes on other time-points}
    		\State remove $u$ from $S$
    	\EndFor	
    \EndWhile
\end{algorithmic}
\end{algorithm}

The pseudo-code of our greedy Kahn's algorithm is listed in Algorithm~\ref{alg:Kahn}. It is easy to see, following the same argument as that of the original Kahn's algorithm, the running time of Greedy Kahn's algorithm is $O(m+n)$. The correctness of the algorithm follows from the following lemma.

\begin{lemma}
For any consistent PA constraint graph $G=(V,E)$, \textsc{Greedy Kahn's Algorithm} generates the minimum normal form timeline for $G$.
\end{lemma}
\begin{proof}
Let $L$ be the normal form timeline generated by \textsc{Greedy Kahn's Algorithm} and let $L'$ be any other normal form timeline. We claim that for any $u\in V$, $L(u)\leq L'(u)$, hence by definition $L$ is the minimum normal form timeline.

Note that \textsc{Greedy Kahn's Algorithm} partitions the time-points in $V$ into $\ell$ disjoint sets, $V= \dot\bigcup_{i=1}^{\ell} S_i$, where $\dot\bigcup$ stands for disjoint set union and $S_i:=\{u\in V: L(u)=i\}$ for every $1\leq i \leq \ell$. We prove the claim by induction on the index number $i$ of $\{S_i\}$ to which a time-point belongs. The claim is certainly true for every time-point in $S_1$. Now for the inductive step, suppose the claim is true for all time-points in $S_1 \cup \cdots \cup S_k$. That is, any other timeline $L'$ must have $L'(u)\geq L(u)$ for every time-point $u$ in $S_1 \cup \cdots \cup S_k$. Let $v$ be a time-point in $S_{k+1}$. Why can not $v$ be in $S_k$? Because there must exist another (or more) time-point $w$ such that $(w, v)\in E$ and $w \in S_k$. So if $L'$ satisfies that $L'(v)=k' \leq k$, then we must also have $L'(w)<k'\leq k=L(w)$, a contradiction. This completes the inductive step of the proof and hence the proof of the lemma.
\end{proof}

Note that \cite{Beek:92} developed an efficient algorithm that finds a consistent scenario (i.e. timeline) for a PA graph by first decomposing the original graph into strongly connected components (SCCs) and then applying topological sort on the simplified PA graph. However, our application of the greedy Kahn's algorithm in solving PA graphs seems to be new.

The timeline for the running example is shown in Figure~\ref{fig:example-timeline}. As before, there are two temporally and aspectually connected subgraphs, where the connection between them (the gray bar) represents the \link{factive} subordinating link. As can be seen, this is a trunk-and-branch multi-timeline structure, where the trunk is (in this case) the ``main'' timeline and the branch is the subordinated timeline.

As we have noted, one major advantage of \TLEX is its ability to properly handle subordinating relations. Prior approaches to timeline extraction ignored subordinating relations. If they were alternatively treated as some sort of temporal relation, machine-learning-based techniques e.g.,\citep{Verhagen:2005,Bethard:2013,Mirza:2016} would likely insert events 8 and 9 between or near time-points $13^-$ and $13^+$. That would indicate events 14 and 15 happened after events 1--12 but before events 17--20, while in the story world, these events did not occur at all. This is clearly incorrect.

\begin{figure}[t]
\begin{center}
  \includegraphics[width=\textwidth]{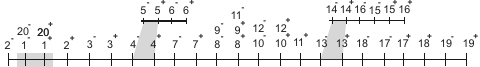}
  \caption{Visualization of the timeline extracted from Figure~\ref{fig:example-constraint-graph}. The three subgraphs are arranged into a main timeline and subordinated timelines connected by a grey branch. The indeterminate sections (20 \& 1) are highlighted.}
  \label{fig:example-timeline}
\end{center} 
\end{figure}

\subsection{Indeterminacy Detection}
\label{sec:indeterminacy}

In most cases, TimeML graphs lack enough information to uniquely specify the full ordering; a simple example is shown in Figure~\ref{fig:old_example-indeterminate-graph}. For that TimeML graph, only the orderings among the first event and the last two events in the timeline are uniquely determined; namely $1^-$\textless $1^+$\textless 2 and 3\textless $4^-$\textless $4^+$\textless $5^-$\textless $5^+$. On the other hand, the ordering between events 2 and event 3 is indeterminate. In fact, there are 13 different possible scenarios, as shown in Figure~\ref{fig:old_indeterminant-list}.

\begin{figure}[ht]
\centering
\includegraphics[width=0.6\columnwidth]{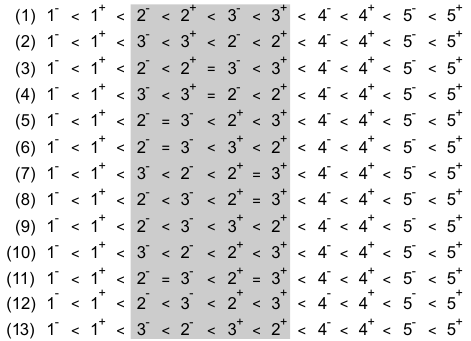}
\caption{All possible solutions for the graph shown in Figure~\ref{fig:old_example-indeterminate-graph}, with indeterminate orderings highlighted.}
\label{fig:old_indeterminant-list}
\end{figure}

As another example, consider the following two events in Example~\ref{ex:running}: John \uline{worked} in the morning and John \uline{ignored} the text in the morning. Both events happened in the morning, however, we do not have enough information to determine these two events' order. It is possible he ignored the text before he started working and it is possible that he ignored it while he was working. Using the DFS algorithm to be developed in Lemma~\ref{lem:indeterminacy-detection}, we can highlight the temporal indeterminacy as shown in Figure~\ref{fig:example-timeline}.

Therefore, in general, we would like to know whether there is any indeterminacy between two time-points in a timeline. Specifically, for any two time-points $u$ and $v$ for which $u<v$ in a given minimum normal form timeline\footnote{We only need to check the case that one time-point is in front of another time-point, since if these two time-points are mapped to the same time $t$ in the minimum normal timeline $L$, then the problem has a trivial solution. To see this, consider the following two possibilities. First, $u$ and $v$ belong to the same ``compound'' time-point. This means that there is an equality constraint between them, therefore these two time-points must always be mapped to the same time in any consistent timeline. Second, $u$ and $v$ do not belong to the same ``compound'' time-point. Since they are mapped to the same time $t$ in the minimum normal timeline, there does not exist any (direct or indirect) constraint between $u$ and $v$. Hence we can construct a timeline in which $u<v$ and another timeline in which $v<u$. } output by Algorithm~\ref{alg:Kahn}, is there another timeline in which $u=v$ or $v<u$? As part of \TLEX, we have an efficient subroutine to check this.

\begin{lemma}
\label{lem:indeterminacy-detection}
Let $L$ be the minimum normal form timeline output by Algorithm~\ref{alg:Kahn} in Step 4 of \TLEX, and let $u$ and $v$ be two time-points in $L$ with $u<v$. Then there is an algorithm with running time $O(m+n)$ which checks whether there is another consistent timeline in which $u=v$ or $v<u$.
\end{lemma}
\begin{proof}
Observe that there exists a timeline $L'$ in which $v=u$ or $v<u$ if and only if there is no directed path from $u$ to $v$ in the PA constraint graph, as otherwise, $u$ must precede $v$ in any consistent timeline. But such a fact can be checked by running the DFS algorithm depicted in Algorithm~\ref{alg:DFS} starting from vertex $u$. Note that the difference between our current algorithm and Algorithm~\ref{alg:DFS} is that here we only run \textsc{DFS-Visit}$(G, u)$ (instead of exhausting all vertices in the graph) and terminate once the the DFS tree rooted at $u$ is completely generated. If the DFS tree does not contain $v$, then there is no such path and there are other consistent timelines in which $u$ and $v$ mapped to the same time or $v$ precedes $u$. Clearly, such a check takes at most the running time of the DFS algorithm, which is $O(m+n)$. 
\end{proof}

Note that, in the worst case that we perform such indeterminacy checking for all pairs of time-points, the total running time of Step 5 of \TLEX is at most $\frac{n(n-1)}{2}\cdot O(m+n)=O(mn^2+n^3)$.

\subsection{End-to-End Main Theorem}

Informally, we call a timeline the ``main'' timeline if it contains events and times that actually happen in the ``world'' described in the text. Because we do not currently have an automatic way of identifying main timelines, we will rely on information external to \TLEX to identify the main timelines. For this, it is sufficient for a person to identify at least one event or time that occurs on every disjoint main timeline for the text.

Before stating our main theorem, we need several definitions.
\theoremstyle{definition}
\begin{definition}[Trunk-and-Branch Timeline]
\label{def:TNBtimeline}
Let $G_{TM} = (V_{TM}, E_{TM})$ be a consistent TimeML graph with temporally and aspectually connected subgraphs $\{G^{(1)}_{TM}, \ldots, G^{(k)}_{TM}\}$ and corresponding timelines $T_L = \{L_1, \ldots, L_k\}$, and let $T_M \subset T_L$ be the set of main timelines. A \emph{trunk-and-branch timeline} for $G_{TM}$ is a tuple $TBT(G_{TM}) = \langle G_{TBT}, T_M \rangle$, where $G_{TBT}$ is the graph obtained from concatenating\footnote{The \emph{concatenation} of two normal form timelines $L_1$ and $L_2$, of graph $G_1=(V_1, E_1)$ and $G_2=(V_2, E_2)$ respectively, is defined in a natural way as follows. Suppose, without loss of generality, that timeline $L_1$ precedes timeline $L_2$, and the length of $L_1$ is $\ell$. Then the concatenated timeline between $L_1$ and $L_2$, denoted $L=L_1 \circ L_2$ with $L: V_1(G_1) \cup V_2(G_2) \to \N$, is defined to be $L(v)=L_1(v)$ if $v \in V_1(G_1)$ and $L(v)=\ell+L_2(v)$ if $v \in V_2(G_2)$. More generally, if $L_1, L_2, \ldots, L_k$ is a sequence of normal form timelines arranged in temporally-increasing order, then their concatenation is defined to be $L=L_1 \circ L_2 \circ \cdots \circ L_k$. } all timelines in $T_L$ together and combined with the set of subordinating TimeML links between timelines;  namely, $V(G_{TBT})=V_{TM}=\cup_{i=1}^{k}V(G^{(i)}_{TM})$,  $E(G_{TBT})=\left(\bigcup_{i=1}^{k} E(L_{i})\right) \bigcup E_S$, and $E_S$ is the set of subordinating links between timelines.
\end{definition}

\theoremstyle{definition}
\begin{definition}[Minimum Normal Form Trunk-and-Branch Timeline]
\label{def:normalTNBtimeline}
Let $TBT(G_{TM}) = \langle G_{TBT}, T_M \rangle$ be a trunk-and-branch timeline for a consistent TimeML graph $G_{TM}$ with temporally connected subgraphs $\{G^{(1)}_{TM}, \ldots, G^{(k)}_{TM}\}$ and corresponding timelines $T_L = \{L_1, \ldots, L_k\}$. We say $TBT(G_{TM})$ is a \emph{normal form trunk-and-branch timeline} if every timeline in  $T_L$ is a normal form timeline. Furthermore, we say $TBT(G_{TM})$ is the \emph{minimum normal form trunk-and-branch timeline} if every timeline in  $T_L$ is a minimum normal form timeline. 
\end{definition}

By combining the algorithms and running time analysis in Sections~\ref{sec:partitioning}, \ref{sec:transforming}, \ref{sec:correcting}, \ref{sec:solving} and \ref{sec:indeterminacy}, we finally arrive at the following end-to-end main theorem of \TLEX:
\begin{theorem}
\label{thm:main}
On an input of TimeML graph $G_{TM} = (V_{TM}, E_{TM})$, with temporally connected subgraphs $\{G^{(1)}_{TM}, \ldots, G^{(k)}_{TM}\}$. Let $V_M \subset V_{TM}$ be a subset of events or times which are identified as on ``main timelines''. Then \TLEX produces one of the two following outputs: 

\begin{itemize}

\item if all $\{G^{(i)}_{TM}\}_{i \in [k]}$ are consistent, \TLEX outputs 
	\begin{enumerate}
        \item the minimum normal-form trunk-and-branch timeline $TBT(G_{TM}) = \langle G_{TBT}, T_M \rangle$ for $G_{TM}$, and every $i\in[k]$ such that $V(L_i) \cap V_M \neq \emptyset$ is identified as a main timeline; and 
 
	   \item a list of indeterminacy tables $\{ID_1, \ldots, ID_k\}$ such that table $ID_i$ stores whether there is indeterminacy between each pair of time-points in timeline $L_i$. 
	\end{enumerate}
 
\item If some $G^{(i)}_{TM}$ is inconsistent, \TLEX outputs a Maximal List of Inconsistent Cycles (MLIC) for each inconsistent temporally connected subgraph.

\end{itemize}

Moreover, the running time of \TLEX is at most $O(n^{2}m+n^{3})$, where $n$ is the number of events or time expressions and $m$ is the number of TimeML links in the input TimeML graph.
\end{theorem}

\section{Experimental Evaluation}
\label{sec:evaluation}

We now discuss our experimental evaluations and show that \TLEX performs as expected on real data. We evaluated \TLEX on the four corpora described previously in Table~\ref{tab:corpora}. As described in \S\ref{sec:correcting}, we ran \textsc{Consistency Checking} across all inconsistent texts and manually corrected those TimeML graphs with reference to the original text. To emphasize the importance of \ALINKs, we performed inconsistency detection on TimeML graphs with and without \ALINKs. Table~\ref{tab:inconsistency} shows the number of inconsistent files in corpora. Without \ALINKs, \TLEX determined that 30 texts across the four corpora had inconsistent TimeML graphs. Importantly, these inconsistencies are exactly the same 30 as those found by \cite{Dercynski:2012}, which checks the consistency of the TimeML graphs by transforming them into PA graphs and applying their closure algorithm. But more importantly, the number of inconsistent texts increases significantly when \ALINKs are included. To correct these inconsistencies, we removed self-loops automatically and manually corrected larger cycles. All manual corrections were by reference to the original texts. We produced fully consistent TimeML graphs for all 385 texts. In some cases, we corrected inconsistencies by removing an incorrect link in the original TimeML graph; in other cases by changing an incorrectly labeled link. The details of the corrections to TimeBank are reported elsewhere~\citep{ocal2022comprehensive}.

\begin{table*}[b]
\centering
\small
\begin{tabular}{l|lll}
\toprule
\bf Corpus   & Total Files & Inconsistent Files & Inconsistent Files  \\
             &             & (\TLINKs only)     & (\TLINKs \& \ALINKs)\\
\midrule
Timebank     & 183         & 30                 & 65 \\
N2 Corpus    & 67          & 10                 & 11 \\
ProppLearner & 15          & 9                  & 11 \\
MEANTIME     & 120         & 36                 & 36* \\
\midrule
Total        & 385         & 85                 & 123 \\
\end{tabular}
\caption{Details of the inconsistencies of the four corpora. *The MEANTIME corpus does not have \ALINKs.}
\label{tab:inconsistency}
\end{table*}

After correction, we extracted the minimum normal form timeline for each text using the \TLEX \textsc{Timeline Generation} step and identified their indeterminate sections using \TLEX \textsc{Indeterminacy Detection}. Table~\ref{tab:timelinedesc} shows statistics on the length of the minimum main timelines (i.e., the number of time steps in the minimum solution), the number of subordinated timelines, and the number of indeterminate sections.

Our experiments took less than a minute for the entire 385 TimeML annotated files from four corpora on a current, standard consumer laptop (2.4 GHz 4-core Intel i7 3630QM with 8GB of RAM). 

\begin{table*}[t]
\centering
\small
\begin{tabular}{l|lll|lll|lll} 
 \toprule
& \multicolumn{3}{c|}{\bf Length of}     & \multicolumn{3}{c|}{\bf \#  Subordinated} & \multicolumn{3}{c}{\bf \# of Indeterminate} \\
& \multicolumn{3}{c|}{\bf Main Timeline} & \multicolumn{3}{c|}{\bf Branches / Text}  & \multicolumn{3}{c}{\bf Sections / Text} \\ 							 
 \bf Corpus	  &	Min & Avg   & Max & Min & Avg & Max & Min & Avg & Max \\
 \midrule
 TimeBank  	  & 4   & 8.1   & 18  & 1  & 16.0 & 142 & 0   & 5.3  & 11\\ 
 N2 Corpus     & 6   & 24.7  & 72  & 1  & 19.3 & 106 & 0   & 10.2 & 21\\
 ProppLearner  & 128 & 180.5 & 290 & 33 & 82.7 & 134 & 0   & 28.8 & 63\\
 MEANTIME 	  & 4   & 7.3   & 13  & 0  & 2.0  & 7   & 0   & 3.8  & 7\\  
\midrule
 Total         & 4   & 17.5  & 290 & 0  & 14.8 & 142 & 0   & 6.6  & 63\\
\end{tabular}
\caption{Details of the timelines extracted from the four
corpora.}
\label{tab:timelinedesc}
\end{table*}

\subsection{Sampling Evaluation}
\label{sec:evalsampling}

As we do not have gold-standard annotated timelines against which to compare, we cannot directly compute recall, precision, or $F_1$ for the output of \TLEX. Therefore we performed a sampling evaluation to check the extracted timelines, where we selected aspects of timelines at random to check against the original texts. We used \textit{Simple Random Sampling} (SRS)~\cite[p.~222]{Saunders:2009} wherein every member of a population has an equally likely chance of being selected. In SRS, we check the correctness of a specific feature of a set of $n$ members randomly selected from a population with size $N$ to obtain an estimate of the correctness of that feature over all the data. The formula for calculating
sample size for a finite population is:

\begin{equation*}
n=\frac{n_0}{1+\frac{n_0-1}{N}} \quad\mathrm{where}\quad n_0=\frac{Z^2}{4c^2} 
\end{equation*}

\indent Where $c$ is the confidence interval, $Z$ is the $Z$-score for the confidence level, and $N$ is the population size. For all experiments, we used a $c$ of $0.02 (2\%)$ and a confidence level of $0.95 (95\%)$, which corresponds to a $Z$-score of $1.96$.

Evaluation experiments were performed by the first author and an independent annotator that we hired. Both the annotator and the first author performed SRS separately, and we calculated the inter-annotator agreement (IAA) score based on their results. We evaluated our extracted timelines via simple random sampling, checking five features: (1) the ordering of time-points with respect to each other, (2) the number of main timelines, (3) the placement of events on main versus subordinated timelines, (4) the connecting point of subordinated timelines, and (5) whether indeterminate sections are truly indeterminate with respect to the underlying TimeML graph.

\subsection{Time-point Ordering}

Time-point ordering was the first feature we evaluated. We randomly picked neighboring pairs of time-points from the \TLEX-generated timelines and compared them to the TimeML graph to ensure that the timeline and the graph were consistent. For example, if we examined a pair of neighboring points labeled A$^+$ and B$^-$, we examined the paths in the TimeML graph between A and B to determine if this ordering was correct. Using SRS, we selected 2,264 pairs of time-points out of 39,534 possible neighboring assignment pairs, sampled from each corpus in proportion to the total number of time-points in each. Both the independent annotator and the first author evaluated the correctness of the ordering, and found each pair to be correctly ordered with respect to the original text (IAA kappa score $\kappa=1$). Since the confidence interval is $2\%$, this results in an estimated accuracy of $\numrange{98}{100}\%$ with $95\%$ confidence. Unlike prior approaches, \TLEX achieves a perfect ordering accuracy, modulo sampling error and correctness of the TimeML graph, using all relation types.

\subsection{Number of Main Timelines}
\label{sec:maintimelines}

Identification of the main timelines requires identification of at least one event or time on each main timeline, which is a laborious process. We approximated this identification, therefore, by identifying as main timelines all timelines that did not have incoming subordination links from another timeline in the trunk-and-branch timeline. 

With this identification in hand, we assessed whether multiple main timelines actually corresponded to temporally disjoint and otherwise unrelatable graphs. In texts with multiple main timelines, we define the {\it breaking pair} between two timelines as the pair of events, times, or an event and time, one from each timeline, which are closest in the text measured by number of intervening words. Because the population size (1,241) was not substantially larger than the sample size (818), we manually checked all breaking pairs in our extracted timelines, and whether they actually indicated disjoint timelines. Both the annotator and the first author observed that all of the breaking pairs corresponded with true disjoint breaks between timelines in the texts (IAA kappa score $\kappa=1$).

For some cases, we observed that sometimes a text had multiple timelines which were disjoint, but should have been collected together into a single main timeline because they were temporally relatable. This is because the annotation of the text is missing a relation between one group of events or times and another group of events or times, but in both cases those events occur in the ``real world'' of the text. In other words, the TimeML graph was disconnected in a way that was inconsistent with the actual semantics of the text, meaning the TimeML annotation was incorrect. We observed this situation for 181 texts, which is another dimension along which the TimeML annotations for these corpora might be improved. In our other work, we investigated the disconnectivity in TimeML graphs and showed that these disconnectivities mainly occur by a missing link between two subgraphs due to a manual annotation error~\citep{ocal2022holistic}. We corrected these erroneous disconnections in other work~\citep{ocal2022comprehensive}.

\subsection{Time-point Placement}

We evaluated whether time-points were correctly placed on main versus subordinated timelines, using our approximation for main timelines as indicated previously. The first author and the annotator checked to see if a time-point placed on a subordinate timeline did not occur in the ``real world'' described in the text. 11,474 time-points were on subordinated timelines, and we sampled 1,986 time-points across the corpora, again in proportion to their number in each corpus. Again, all samples were correct (IAA kappa score $\kappa=1$), giving an estimated accuracy of $\numrange{98}{100}\%$ with $95\%$ confidence.

\subsection{Subordinated Connecting Points}

We checked to confirm that subordinated timelines were connected to the correct time-points on the identified main timelines. There are 5,701 subordinated timelines in our data, and every one was connected to a main timeline with at least one subordinating link. The first author and the annotator checked 1,690 connections, observing that every connection was correctly placed (IAA kappa score $\kappa=1$), giving an estimated accuracy of $\numrange{98}{100}\%$ with $95\%$ confidence.

\subsection{Indeterminate Sections}
\label{sec:samplingind}

Finally, the first author and the annotator checked whether indeterminate sections were truly indeterminate with respect to the TimeML graph. There are 2,541 indeterminate sections, comprising 11,688 pairs of time-points. We randomly selected 1,992 time-point pairs from the indeterminate sections. In all cases the pairs were truly indeterminate (IAA kappa score $\kappa=1$), meaning an estimated accuracy of $\numrange{98}{100}\%$ with a $95\%$ confidence.

\section{Future Work}
\label{sec:futurework}

Our work naturally suggests several interesting directions of future research.

Our investigation showed that manually annotated TimeML texts often contain errors that result in temporal inconsistencies; automatic methods for generating TimeML annotations are currently even more noisy and error-prone. Because of this, a logical next step is to develop methods that automatically suggest corrections to the maximal list of inconsistent cycles identified in the \textsc{Consistency Checking} step. Both rule-based and supervised machine learning-based approaches could be developed to address this problem.

Finally, it is interesting to investigate the performance of \TLEX on automatically generated TimeML graphs. Indeed, \TLEX may be applied in this way to measure the quality of the timelines extracted, which provides a good indication of how useful automatic TimeML extraction is in practice.

\section{Contributions}
\label{sec:contributions}

In this work, we developed \TLEX (TimeLine EXtraction), an exact, end-to-end solution which extracts timelines from TimeML annotated texts. The novel features of \TLEX  include identifying specific relations that are involved in an inconsistency (which can then be manually corrected) and determining sections of the timelines that have indeterminate order, which is critical for downstream tasks such as aligning events from different timelines. We provide efficient algorithms which implement the tasks within \TLEX and conduct experimental evaluations by applying \TLEX to 385 TimeML annotated texts from four corpora. We provide a reference implementation of \TLEX, the extracted timelines for all texts, and the manual corrections to the temporal graphs of the inconsistent texts.


\section{Acknowledgments}
We gratefully acknowledge valuable discussions with Jared Hummer, Antonela Radas, Victor Yarlott, Karine Megerdoomian, Akul Singh, and Emmanuel Garcia. Contributions to the implementation of the TLEX algorithm in Java and python were made by FIU KFSCIS Senior Project team members Christopher Eberhard, Stephan Belizaire, Pablo Maldonado, Luis Robaina, Adrian Silva, Victoria Fernandez, Ronald Pena, Ismael Clark, Felipe Arce, Kevin Fontela, Raul Garcia, Leandro Estevez, Carlos Pimentel, Hector Borges, Ivan Parra Sanz, Tony Erazo, Sage Pages, Gerardo Parra, and Franklin Bello Romero. This work was partially funded by research grants TO-134841, TO-135998, and TO-139837 to Dr. Finlayson from MITRE Corporation. Any opinions, findings, and conclusions or recommendations expressed in this material are those of the author(s) and do not necessarily reflect the views of MITRE Corporation. This work was also partially supported by ONR Award N00014-17-1-2983 to Dr. Finlayson.

\bibliographystyle{unsrtnat}
\bibliography{main}  






\end{document}